\documentclass[twoside]{article}
\usepackage[accepted]{aistats2025_arxiv}

\usepackage[utf8]{inputenc} %
\usepackage[T1]{fontenc}    %
\usepackage{url}            %
\usepackage{booktabs}       %
\usepackage{amsfonts}       %
\usepackage{nicefrac}       %
\usepackage{microtype}      %

\usepackage[linesnumbered,ruled,vlined]{algorithm2e}
\usepackage{amsmath}
\usepackage{amssymb}
\usepackage{amsthm}
\usepackage{bbm}
\usepackage{bm}
\usepackage{color}
\usepackage{dirtytalk}
\usepackage{dsfont}
\usepackage{enumerate}
\usepackage{graphicx}
\usepackage{listings}
\usepackage{mathtools}
\usepackage[numbers]{natbib}
\usepackage{subfigure}
\usepackage{times}
\usepackage{xspace}
\usepackage{enumitem}

\usepackage[usenames,dvipsnames]{xcolor}
\usepackage[bookmarks=false]{hyperref}
\hypersetup{
pdffitwindow=true,
pdfstartview={FitH},
pdfnewwindow=true,
colorlinks,
linktocpage=true,
linkcolor=Green,
urlcolor=Green,
citecolor=Green
}
\usepackage[capitalize,noabbrev]{cleveref}

\usepackage{thmtools}
\declaretheorem[name=Theorem,refname={Theorem,Theorems},Refname={Theorem,Theorems}]{theorem}

\declaretheorem[name=Corollary,refname={Corollary,Corollaries},Refname={Corollary,Corollaries},sibling=theorem]{corollary}
\declaretheorem[name=Assumption,refname={Assumption,Assumptions},Refname={Assumption,Assumptions}]{assumption}
\declaretheorem[name=Proposition,refname={Proposition,Propositions},Refname={Proposition,Propositions},sibling=theorem]{proposition}
\declaretheorem[name=Definition,refname={Definition,Definitions},Refname={Definition,Definitions},sibling=theorem]{definition}

\newcommand{\cA}{\mathcal{A}}

\newcommand{\cC}{\mathcal{C}}

\newcommand{\cI}{\mathcal{I}}

\newcommand{\cK}{\mathcal{K}}
\newcommand{\cL}{\mathcal{L}}

\newcommand{\cR}{\mathcal{R}}
\newcommand{\cS}{\mathcal{S}}

\newcommand{\cV}{\mathcal{V}}

\newcommand{\cX}{\mathcal{X}}
\newcommand{\cY}{\mathcal{Y}}
\newcommand{\cZ}{\mathcal{Z}}

\newcommand{\realset}{\mathbb{R}}

\newcommand{\prob}[1]{\mathbb{P} \left(#1\right)}

\newcommand{\abs}[1]{\left|#1\right|}

\newcommand{\I}[1]{\mathds{1} \! \left\{#1\right\}}

\newcommand{\normw}[2]{\|#1\|_{#2}}

\newcommand{\set}[1]{\left\{#1\right\}}

\newcommand{\T}{^\top}
\newcommand{\eqdef}{\overset{\Delta}{=}}

\DeclareMathOperator*{\argmax}{arg\,max\,}
\DeclareMathOperator*{\argmin}{arg\,min\,}
\let\det\relax
\DeclareMathOperator{\det}{det}

\mathchardef\mhyphen="2D

\newcommand{\ouralgo}{\ensuremath{\tt DopeWolfe}\xspace}
\newcommand{\partialgradnox}{\ensuremath{\tt PartialGrad}}
\newcommand{\partialgrad}{\partialgradnox\xspace}
\newcommand{\updateinversenox}{\ensuremath{\tt UpdateInverse}}
\newcommand{\updateinverse}{\updateinversenox\xspace}
\newcommand{\updatelogdetnox}{\ensuremath{\tt UpdateLogDet}}
\newcommand{\updatelogdet}{\updatelogdetnox\xspace}
\newcommand{\goldensearchnox}{\ensuremath{\tt GoldenSearch}}
\newcommand{\goldensearch}{\goldensearchnox\xspace}

\begin{document}

\runningtitle{Comparing Few to Rank Many: Active Human Preference Learning using Randomized Frank-Wolfe}

\runningauthor{Kiran Koshy Thekumparampil, Gaurush Hiranandani, Kousha Kalantari, Shoham Sabach, Branislav Kveton}

\twocolumn[

\aistatstitle{Comparing Few to Rank Many: \\ Active Human Preference Learning using Randomized Frank-Wolfe}

\aistatsauthor{Kiran Koshy Thekumparampil$^*$ \And Gaurush Hiranandani$^*$}

\aistatsaddress{Search, Amazon \And Typeface} 

\aistatsauthor{Kousha Kalantari \And Shoham Sabach \And Branislav Kveton}

\aistatsaddress{AWS AI Labs, Amazon \And AWS AI Labs, Amazon \And Adobe Research } 

]

\begin{abstract}
We study learning of human preferences from a limited comparison feedback. This task is ubiquitous in machine learning. Its applications such as reinforcement learning from human feedback, have been transformational. We formulate this problem as learning a Plackett-Luce model over a universe of $N$ choices from $K$-way comparison feedback, where typically $K \ll N$. Our solution is the D-optimal design for the Plackett-Luce objective. The design defines a data logging policy that elicits comparison feedback for a small collection of optimally chosen points from all ${N \choose K}$ feasible subsets. The main algorithmic challenge in this work is that even fast methods for solving D-optimal designs would have $O({N \choose K})$ time complexity. To address this issue, we propose a randomized Frank-Wolfe (FW) algorithm that solves the linear maximization sub-problems in the FW method on randomly chosen variables. We analyze the algorithm, and evaluate it empirically on synthetic and open-source NLP datasets.
\end{abstract}

\section{Introduction}
\label{sec:introduction}

Learning to rank from human feedback is a fundamental problem in machine learning. In web search \citep{agichtein06learning,hofmann13fidelity,hofmann16online}, human annotators rate responses of the search engine, which are then used to train a new system. More recently, in \emph{reinforcement learning from human feedback (RLHF)} \citep{ouyang22training,rafailov23direct}, human feedback is used to learn a latent reward model, which is then used to align a \emph{large-language model (LLM)} with human preferences. Both discussed problems involve a potentially large universe of $N$ items, and the human can provide feedback only on a small subset of them. In web search, the items are the retrieved web pages for a given query and the annotators can only label a top few. In RLHF, the items are LLM responses and the feedback is typically pairwise \citep{ouyang22training,rafailov23direct}, although $K$-way extensions have also been studied \cite{zhu23principled}. In this work, we answer the question what human feedback to collect to learn a good model of human preferences for our motivating problems.

We formulate the problem as learning to rank $N$ items, such as web pages or LLM responses, from a limited $K$-way human feedback. In general, $K \ll N$ because humans cannot provide high-quality preferential feedback on a large number of choices \citep{tversky74judgment}. When $K = 2$, we have a relative feedback over two responses, known as the \emph{Bradley-Terry-Luce (BTL)} model \citep{bradley52rank}. When $K \geq 2$, we obtain its generalization to a ranking feedback over $K$ responses, known as the \emph{Plackett-Luce (PL)} model \citep{plackett75analysis,luce05individual}. To learn high-quality preference models, we ask humans questions that maximize information gain. Such problems have long been studied in the field of optimal design \cite{pukelsheim93optimal,boyd04convex,lattimore19bandit}. An optimal design is a distribution over most informative choices that minimizes uncertainty in some criterion. This distribution can be generally obtained by solving an optimization problem and often has some desirable properties, like sparsity. The focus of our work is on solving large optimal designs, with billions of potential choices for $K$-way feedback.

We make the following contributions:

\textbf{(1)} We propose a general framework for collecting human feedback to learn to rank $N$ items from $K$-way feedback (\cref{sec:setting,sec:algorithm}), where $K \leq N$. This framework is more general than in prior works \citep{mukherjee24optimal,mehta23sample,das24active}.

\textbf{(2)} We bound the generalization error of the model learned from human feedback and analyze the resulting ranking loss (\cref{sec:analysis}). The generalization error and ranking loss decrease with more human feedback.

\textbf{(3)} The main algorithmic challenge in this work is the time complexity of computing the policy for collecting human feedback. The policy is a solution to the D-optimal design \cite{kiefer1960equivalence}. While fast methods exist for solving this problem, using the Frank-Wolfe (FW) method \citep{frank1956algorithm,dunn1978conditional}, the time complexity of the linear maximization oracle in the FW method is $O({N \choose K})$ in our setting (\cref{sec:frank-wolfe method}). To address this issue, we propose a randomized Frank-Wolfe algorithm \ouralgo that solves the linear maximization sub-problems in the FW method on randomly chosen point (\cref{sec:scalable algorithm}). To further speed up computations, we propose caching strategies for elements of the gradient and using the golden-section search, low-rank and sparse updates.

\textbf{(4)} We analyze the randomized FW algorithm in \cref{thm:rand_fw,cor:ouralgo}. The main technical novelty in our analysis is generalizing analyses of a randomized FW method beyond Lipschitz smoothness.

\textbf{(5)} We confirm all findings empirically (\cref{sec:experiments}), from computational gains to a lower sample complexity for $K$-way human feedback with a larger $K$. Some improvements, such as the sample size reduction on the Nectar dataset, are an order of magnitude. We experiment with both synthetic and real-world problems. The real-world datasets represent our both motivating examples, learning to rank and learning the reward model in RLHF.

\textbf{Related work:} There are two lines of related works: learning to rank from human feedback and solving large-scale optimal designs. \citet{mehta23sample} and \citet{das24active} learn to rank pairs of items from pairwise feedback. They optimize the maximum gap. \citet{mukherjee24optimal} learn to rank lists of $K$ items from $K$-way feedback. They optimize both the maximum prediction error and ranking loss. We propose a general framework for learning to rank $N$ items from $K$-way feedback, where $K \geq 2$ and $N \geq K$. This setting is more general than in the above prior works. Our main algorithmic contribution is a randomized Frank-Wolfe algorithm. It addresses a specific computational problem of our setting, solving a D-optimal design with $O({N \choose K})$ variables for learning to rank. This distinguishes it from other recent works for solving large-scale optimal designs \citep{hendrych2023solving,ahipacsaouglu2015first,zhao2023analysis}. We review related works in more detail in \cref{sec:related work}.

This paper is organized as follows. In \cref{sec:setting}, we introduce our ranking problem. In \cref{sec:algorithm}, we present a generic framework for learning to rank $N \gg K$ items from $K$-way feedback. We introduce the FW algorithm in \cref{sec:frank-wolfe method} and make it efficient in \cref{sec:scalable algorithm}. Our experiments are in \cref{sec:experiments}. We conclude in \cref{sec:conclusions}.

\section{Setting}
\label{sec:setting}

We start with introducing notation. Let $[n] = \set{1, \dots, n}$ and $\mathbb{R}^d$ is the $d$ dimensional real vector space. Let $\triangle^S$ be the probability simplex over set  $S$. For any distribution $\pi \in \triangle^S$, we have $\sum_{i \in S} \pi(i) = 1$. Let $\normw{x}{A} = \sqrt{x\T A x}$ for any positive-definite matrix $A \in \realset^{d \times d}$ and vector $x \in \realset^d$. We say that a square matrix $M \in \mathbb{R}^{d \times d}$ is PSD/PD if it is positive semi-definite/positive definite.

We consider a problem of learning to rank $N$ items. An item $k \in [N]$ is represented by its feature vector $x_k \in \cX$, where $\cX \subseteq \realset^d$ is the set of all feature vectors. The relevance of an item is given by its mean reward. The mean reward of item $k$ is $x_k\T \theta_*$, where $\theta_* \in \realset^d$ is an unknown parameter. Such models of relevance have a long history in learning to rank \citep{zong16cascading,li16contextual}. Without loss of generality, we assume that the original order of the items is optimal, $x_j\T \theta_* > x_k\T \theta_*$ for any $j < k$.

We interact with a human $T$ times. In interaction $\ell \in [T]$, we select a $K$-subset of items $S_\ell \in \cS$ and the human ranks all items in $S_\ell$ according to their preferences, where $\cS$ is a collection of all $K$-subsets of $[N]$. Note that $\abs{\cS} = {N \choose K}$. We represent the ranking as a permutation $\sigma_\ell: [K] \rightarrow S_\ell$, where $\sigma_\ell(k)$ is the item at position $k$. The probability that this permutation is generated is
\begin{align}
  \prob{\sigma_\ell}
  = \prod_{k = 1}^K \frac{\exp[x_{\sigma_\ell(k)}\T \theta_*]}
  {\sum_{j = k}^K \exp[x_{\sigma_\ell(j)}\T \theta_*]}\,,
  \label{eq:ranking feedback}
\end{align}
In short, items with higher mean rewards are more preferred by humans and thus more likely to be ranked higher. This feedback model is known as the \emph{Plackett-Luce (PL)} model \citep{plackett75analysis,luce05individual,zhu23principled}, and it is a standard model for learning unknown relevance of choices from relative feedback.

Our goal is to select the subsets such that we can learn the true order of the items. Specifically, after $T$ interactions, we output a permutation $\hat{\sigma}: [N] \to [N]$, where $\hat{\sigma}(k)$ is the index of the item placed at position $k$. The quality of the solution is measured by the \emph{ranking loss}
\begin{align}
  R(T)
  = \frac{2}{N (N - 1)} \sum_{j = 1}^N \sum_{k = j + 1}^N
  \I{\hat{\sigma}(j) > \hat{\sigma}(k)}\,,
  \label{eq:ranking loss}
\end{align}
where $N (N - 1) / 2$ is a normalizing factor that scales the loss to $[0, 1]$. Simply put, the ranking loss is the fraction of incorrectly ordered pairs of items in permutation $\hat{\sigma}$. It can also be viewed as the normalized Kendall tau rank distance \cite{kendall1948rank} between the optimal order of items and that according to $\hat{\sigma}$, multiplied by $2 / (N (N - 1))$. While other objectives are possible, such as the \emph{mean reciprocal rank (MRR)} and \emph{normalized discounted cumulative gain (NDCG)} \citep{manning08introduction}, we focus on the ranking loss in \eqref{eq:ranking loss} to simplify our presentation. For completeness, we report the NDCG in our experiments.

To simplify exposition, and without loss of generality, we focus on sorting $N$ items in a single list. Our setting can be generalized to multiple lists as follows. Suppose that we want to sort multiple lists with $N_1, \dots, N_M$ items using $K$-way feedback. This can be formulated as a sorting problem over $\abs{N_1 \choose K} + \dots + \abs{N_M \choose K}$ subsets, containing all $K$-subset from all original lists. The subsets can be hard to enumerate, because their number is exponential in $K$. This motivates our work on large-scale optimal design for human feedback.

\section{Optimal Design for Learning To Rank}
\label{sec:algorithm}

This section presents our algorithm and its generalization analysis. These generalize the work of \citet{mukherjee24optimal} to ranking $N$ items from $K$-way feedback, from ranking lists of length $K$ only. The fundamental novel challenge is the computational cost of solving \eqref{eq:optimal design}. In particular, since $\abs{\cS} = {N \choose K}$, \eqref{eq:optimal design} has exponentially many variables in $K$. Therefore, it cannot be written out and solved using standard tools, as was done using CVXPY \citep{diamond2016cvxpy} in \citet{mukherjee24optimal}. We address this challenge separately in \cref{sec:scalable algorithm}.

\subsection{Optimal Design}
\label{sec:optimal design}

Suppose that the human interacts with $T$ subsets $\{S_\ell\}_{\ell \in [T]}$ (\cref{sec:setting}). We use the human responses $\sigma_\ell$ on $S_\ell$ to estimate $\theta_*$. Specifically, since the probability of $\sigma_\ell$ under $\theta_*$ is \eqref{eq:ranking feedback}, the negative log-likelihood of all feedback is
\begin{align}
  \cL_T(\theta)
  = - \sum_{\ell = 1}^T \sum_{k = 1}^K
  \log\left(\frac{\exp[x_{\sigma_\ell(k)}\T \theta]}
  {\sum_{j = k}^K \exp[x_{\sigma_\ell(j)}\T \theta]}\right)\,.
  \label{eq:ranking loglik}
\end{align}
To estimate $\theta_*$, we solve a \emph{maximum likelihood estimation (MLE)} problem, $\hat{\theta} = \argmin_{\theta \in \Theta} \cL_T(\theta)$. The problem \eqref{eq:ranking loglik} can be solved by gradient descent with standard or adaptive step sizes \cite{negahban2018learning}. Finally, we estimate the mean reward of item $k$ as $x_k\T \hat{\theta}$ and sort the items in descending order of $x_k\T \hat{\theta}$, which defines $\hat{\sigma}$ in \eqref{eq:ranking loss}.

The problem of selecting informative subsets $S_\ell$ is formulated as follows. For any $z \in \realset^d$ and PSD matrix $M \in \realset^{d \times d}$, the prediction error at $z$ can be bounded using the Cauchy-Schwarz inequality as
\begin{align}
  |z\T (\hat{\theta} - \theta_*)|
  \leq \normw{z}{M^{-1}} \normw{\hat{\theta} - \theta_*}{M}\,.
  \label{eq:prediction error}
\end{align}
It remains to choose an appropriate matrix $M$. Following \citet{zhu23principled}, we set
\begin{align}
  M
  = \frac{e^{- 4}}{2 K (K - 1)}
  \sum_{\ell = 1}^T \sum_{j = 1}^K \sum_{k = j + 1}^K z_{\ell, j, k} z_{\ell, j, k}\T\,,
  \label{eq:ranking hessian}
\end{align}
where $z_{\ell, j, k} = x_{\sigma_\ell(j)} - x_{\sigma_\ell(k)}$ is the difference between feature vectors of items $j$ and $k$ in interaction $\ell$. We make an assumption to relate $M$ to the Hessian of $\cL_T(\theta)$.

\begin{assumption}
\label{ass:unit norms} For all $k \in [N]$, $\normw{x_k}{2} \leq 1$. In addition, $\normw{\theta_*}{2} \leq 1$ and $\normw{\hat{\theta}}{2} \leq 1$.
\end{assumption}

These kinds of assumptions are standard in the analysis of MLE of PL models \cite{negahban2018learning}.
Under \cref{ass:unit norms}, $\nabla^2 \cL_T(\hat{\theta}) \succeq M$, where $\nabla^2 \cL_T(\theta)$ is the Hessian of $\cL_T(\theta)$. Since $M$ is PSD, $\cL_T$ is strongly convex at $\hat{\theta}$, and this gives us a $\tilde{O}(\sqrt{d})$ high-probability upper bound on the second term in \eqref{eq:prediction error} \citep{zhu23principled}. To control the first term, we note that the minimization of $\max_{z \in \cZ} \normw{z}{M^{-1}}$, where $\cZ$ is the set of all feature vector differences, is equivalent to maximizing $\log\det(M)$ \citep{kiefer1960equivalence}. This maximization problem is known as the D-optimal design.

We solve the problem as follows. Each subset $S \in \cS$ is represented by a matrix $A_S$ defined as $A_S = (z_{j, k})_{(j, k) \in \Pi_2(S)}$, where $z_{j, k} = x_j - x_k$ is the difference of feature vectors of items $j$ and $k$, and
\begin{align}
  \Pi_2(S) = \{(j, k): j < k; j, k \in S\}
  \label{eq:pairs}
\end{align}
is the set of all pairs in $S$ where the first entry has lower index than the second one. Thus matrix $A_S$ has $d$ rows and $K (K - 1) / 2$ columns. Equipped with these matrices, we solve the following (determinant) D-optimal design \cite{pukelsheim93optimal}
\begin{align}
  \pi_*
  & = \argmax_{\pi \in \Delta^\cS} g(\pi)\,, \text{ where }
  \label{eq:optimal design} \\
  g(\pi)
  & = \log\det\left(V_\pi \right)\,, \text{ and }\;\;
  V_\pi = \sum_{S \in \cS} \pi(S) A_S A_S\T,
  \nonumber
\end{align}
where $\pi \in \Delta^\cS$ is a probability distribution over the subsets in $\cS$, $\Delta^\cS$ is the simplex of all such distributions, and $\pi(S)$ denotes the probability of choosing the subset $S$ under $\pi$. Note that we could have indexed $\pi$ by an integer defined through a bijective mapping $\cC: \mathcal{S} \to [{N\choose K}]$ from the subsets to natural numbers. We do not do this to simplify presentation. The optimization problem is concave since $\log\det$ is concave for PSD matrices and all $\pi(S) A_S A_S$ are PSD by design. Further, its solution is sparse \citep{kiefer1960equivalence}. Therefore, fast convex optimization methods, such as the Frank-Wolfe method, can be used to solve it. After $\pi_*$ is computed, the human feedback is collected on subsets sampled from $\pi_*$, i.e.~$S_\ell \sim \pi_*$.

\subsection{Generalization Analysis}
\label{sec:analysis}

Here we provide the generalization guarantee for the D-optimal design. We start with proving that the differences in estimated item relevance under $\hat{\theta}$ converge to those under $\theta_*$ as the sample size $T$ increases. The proof is under two assumptions that are not essential and only avoiding rounding. We also make \cref{ass:unit norms} throughout this section.

\begin{proposition}
\label{prop:prediction error} Let $K$ be even and $N / K$ be an integer. Let the feedback be collected according to $\pi_*$ in \eqref{eq:optimal design}. Then with probability at least $1 - \delta$,
\begin{align*}
  \sum_{j = 1}^N \sum_{k = j + 1}^N
  \left(z_{j, k}\T (\theta_* - \hat{\theta})\right)^2
  = \tilde{O}\left(\frac{N^2 K^4 d^2 \log(1 / \delta)}{T}\right)\,.
\end{align*}
\end{proposition}
\begin{proof}
We build on Theorem 5.3 of \citet{mukherjee24optimal}, which says that for any collection $\cS$ of $K$-subsets,
\begin{align*}
  \max_{S \in \cS} \sum_{(j, k) \in \Pi_2(S)}
  \left(z_{j, k}\T (\theta_* - \hat{\theta})\right)^2
  = \tilde{O}\left(\frac{K^6 d^2 \log(1 / \delta)}{T}\right)
\end{align*}
holds with probability at least $1 - \delta$. To reuse this result, we design a collection $\cC$ such that any item pair appears in at least one $S \in \cC$. Then any $(z_{j, k}\T (\theta_* - \hat{\theta}))^2$ would be bounded. Let $\cI = \set{I_\ell}_{\ell \in [2 N / K]}$ be a partition of $[N]$ into sets with consecutive item indices, each of size $K / 2$. Then $\cC$ is designed as follows. The first $2 N / K - 1$ sets in $\cC$ contain items $I_1$ combined with any other set $\cI \setminus I_1$, the next $2 N / K - 1$ sets in $\cC$ contain items $I_2$ combined with any other set $\cI \setminus I_2$, and so on. Clearly, the size of $\cC$ is at most $4 N^2 / K^2$ and all $(z_{j, k}\T (\theta_* - \hat{\theta}))^2$ are covered. Thus
\begin{align*}
  \sum_{j = 1}^N \sum_{k = j + 1}^N
  \left(z_{j, k}\T (\theta_* - \hat{\theta})\right)^2
  \leq \frac{4 N^2}{K^2} \tilde{O}\left(\frac{K^6 d^2 \log(1 / \delta)}{T}\right)\,.
\end{align*}
This concludes the proof.
\end{proof}

The bound in \cref{prop:prediction error} is $O(d^2 / T)$, which is standard for a squared prediction error in linear models with $d$ parameters and sample size $T$. The dependencies on $N^2$, $K^4$, and $\log(1 / \delta)$ are due to bounding predictions errors of $O(N^2)$ item pairs, from relative $K$-way feedback with probability at least $1 - \delta$.

Now we derive an upper bound on the ranking loss $R(T)$.

\begin{proposition}
\label{prop:ranking loss} Let the feedback be collected according to $\pi_*$ in \eqref{eq:optimal design}. Then the ranking loss is bounded as
\begin{align*}
  R(T)
  \leq \frac{4}{N (N - 1)} \sum_{j = 1}^N \sum_{k = j + 1}^N
  \exp\left[- \frac{T \kappa^2 z_{j, k}^2}{2 d}\right]\,,
\end{align*}
where $\kappa > 0$ (Assumption 2 in \citet{mukherjee24optimal}).
\end{proposition}
\begin{proof}
We build on Theorem 5.4 of \citet{mukherjee24optimal}. Specifically, the key step in their proof is that
\begin{align*}
  \prob{x_j\T \hat{\theta} \leq x_k\T \hat{\theta}}
  \leq 2 \exp[- T \kappa^2 z_{j, k}^2 / (2 d)]
\end{align*}
holds for any set of $K$ items $S$ and items $(j, k) \in \Pi_2(S)$ in it. Our claim follows from noting that $\prob{\hat{\sigma}(j) > \hat{\sigma}(k)} = \prob{x_j\T \hat{\theta} \leq x_k\T \hat{\theta}}$ and then applying the above bound.
\end{proof}

\cref{prop:ranking loss} says that the ranking loss decreases exponentially with sample size $T$, $\kappa^2$, and squared gaps; and increases with the number of features $d$. This is the same as in Theorem 2 of \citet{azizi22fixedbudget} for fixed-budget best-arm identification in GLMs. Therefore, although we do not prove a lower bound, our bound is likely near-optimal.

\section{Frank-Wolfe Method for Optimal Design}
\label{sec:frank-wolfe method}

Historically, the Frank-Wolfe (FW) method has been utilized as a scalable algorithm for solving the traditional D-optimal design problem \cite{khachiyan1996rounding,zhao2023analysis}.
When instantiating the FW method for our problem \eqref{eq:optimal design} we get \cref{algo:fw}.
FW consists of primarily two high-level steps. First, it computes the gradient $G_t$ at the current iterate (\cref{algo-step:fw_grad} of \cref{algo:fw}) and then it finds the distribution $\widehat{\pi}_t$ which maximizes the linear functional defined by $G_t$, i.e.~$\max_{\pi \in \Delta^\mathcal{S}} \langle G_t, \pi \rangle$ (\cref{algo-step:fw_lmo} of \cref{algo:fw}) using a \emph{Linear Maximization Oracle (LMO)}. 
Finally, it updates the iterate with a convex combination of the current iterate $\pi_t$ and $\widehat{\pi}_t$ with a stepsize $\alpha_t$ chosen so that it maximizes $\max_\alpha g((1-\alpha) \cdot \pi_t + \alpha \cdot \widehat{\pi}_{t})$ (\cref{algo-step:fw_line_search} of \cref{algo:fw}) using a \emph{linear search} algorithm. 
It is known that this method converges to the maximizer of \eqref{eq:optimal design} \cite{zhao2023analysis}.

\begin{algorithm}[htbp]
\caption{Frank-Wolfe Method on $\Delta^\mathcal{S}$ simplex}
\label{algo:fw}
\KwIn{Objective $g$, \#steps $T_{\mathrm{od}}$, initial iterate $\pi_0 \in \Delta^\mathcal{S}$}
\For{$t = 0, 1, \ldots, T_{\mathrm{od}}$}{
    Compute gradient: $G_t = \nabla_\pi g(\pi_t)$ \label{algo-step:fw_grad} 

    Linear Max.~Oracle (LMO): $\widehat{\pi}_{t} \in \argmax\limits_{\pi \in \Delta^\mathcal{S}} \langle G_t, \pi \rangle$ \label{algo-step:fw_lmo} 
    
    Line search: $\alpha_t \in \argmax\limits_{\alpha \in [0, 1]} g((1-\alpha) \cdot \pi_t + \alpha \cdot \widehat{\pi}_{t})$ \label{algo-step:fw_line_search}
    
    Iterate update: $\pi_{t+1} = (1 - \alpha_t) \cdot \pi_t + \alpha_t \cdot \widehat{\pi}_{t}$ \label{algo-step:fw_iterate}
}
\Return $\pi_{T_{\mathrm{od}}}$

\end{algorithm}

Typically, FW method is used when there is a conceptually simple LMO. In fact, for our problem \eqref{eq:optimal design} on the simplex $\Delta^\cS$, we can simplify the LMO as
\begin{align}
\widehat{\pi}_t &= e_{S_t}\,, \text{ where } \; S_t = \argmax_{S \in \cS} G_t(S)\,, \text{ and } \nonumber \\
e_{S_t}(S) &= \mathbf{1}\{S = S_t\}\,, \forall S \in \cS.
\label{eq:fw_lmo}
\end{align}
That is $\widehat{\pi}_t$ has all its probability mass on a set $S_t$ with the largest corresponding partial derivative $\max_{S} G_t(S)$. This also implies that the line search step with $e_{S_t}$ can only increase the number of non-zero elements in the iterate by at most 1. Thus after $t$ steps, the sparsity can only decrease by $t/{N \choose K}$. This is desirable as we prefer that the optimal design chooses only a small number of subsets for eliciting ranking feedback. This is one of the reasons why FW method is preferred over the standard gradient descent.

While LMO step appears simple, it requires computing the gradient w.r.t.~an ${N \choose K}$ dimensional vector and finding its maximimum entry. When $N \gg K$ this has an exponential computational complexity $\Omega(N^K)$ in $K$. The line search is also computationally expensive because it requires calculating the objective at many feasible iterates and the objective calculation involves computing the $\log\det$ of a $d \times d$ matrix which requires $O(d^3)$ operations in most practical implementations. In \cref{sec:scalable algorithm} we  address these and other limitations through a new algorithm, \ouralgo with a better per-iteration complexity in terms of both $K$ and $d$.

\section{\ouralgo: Randomized Frank-Wolfe with Low-Rank and Sparse Updates}
\label{sec:scalable algorithm}

\begin{algorithm}[h]
\caption{\ouralgo for solving D-optimal design for selecting subsets for eliciting $K$-way ranking feedback}
\label{algo:ouralgo}
\KwIn{\#items $N$, subset size $K$, features $\{x_i\}_{i=1}^N$, sampling size $R$, \#steps $T_{\mathrm{od}}$, initial iterate $\pi_0 \in \Delta^\mathcal{S}$, $\log\det$ argument matrix $V_{\pi_0}$, 
regularization $\gamma$, and line-search tolerance $\alpha_\mathrm{tol}$}
Let $\mathcal{S}$ be the collection of all $K$-sized subsets of $[N]$

Set $z_{j,k} = x_j - x_k$, $\forall (j,k) \in \Pi_2([N])$

Set $V_0^{\mathrm{inv}} = (V_{\pi_0} + \gamma \mathbf{I}_{d \times d})^{-1}$ \label{algo-step:v_inv}%

\For{$t = 0, 1, \ldots, T$}{
    \tcp{Randomized LMO}
    Sample $\mathcal{R}_t \sim \mathrm{Uniform}(\{\mathcal{R} \subseteq \mathcal{S} \mid |\mathcal{R}| = R \rceil\})$ \label{algo-step:sample} 

    $\{G_t(S)\} = \text{\partialgrad}(N, {\mathcal{R}}_t, V^{\mathrm{inv}}_t, \{z_{j,k}\})$ \label{algo-step:partial_grad} 
    
    $S_t \in \argmax_{S \in \mathcal{R}_t} G_t(S)$ \label{algo-step:lmo} 

    \tcp{Line search (\cref{algo:ouralgo_subroutines})}
    $\alpha_t = \text{\goldensearch}(V_t^{\mathrm{inv}}, A, \alpha_{\mathrm{tol}})$  \label{algo-step:golden-search}
    
    \tcp{Update iterate (\cref{algo:ouralgo_subroutines})}
    $\pi_{t+1} = (1 - \alpha_t) \cdot \pi_t + \alpha_t \cdot e_{{S}_{t}}$ \label{algo-step:update-iterate}

    $V_{t+1}^{\mathrm{inv}} = \text{\updateinverse}(V_{t}^{\mathrm{inv}}, A_S, \alpha_t)$ \label{algo-step:update-inv}
    
}
\Return $\pi_T$

\SetKwFunction{FMain}{\partialgradnox}
\SetKwProg{Fn}{Sub-routine}{:}{}
\Fn{\FMain{$N, \mathcal{R}, V^{\mathrm{inv}}, \{z_{j,k}\}$}\label{algo-step:partial_grad_subroutine}}{
    \For{$(j,k) \in \Pi_2([N])$}{
        \tcp{Pair gradients}
        $D_{j,k} \gets z_{j,k}^\top V^{\mathrm{inv}} z_{j,k}$ \label{algo-step:pair_grad}
    }

    \For{$S \in \mathcal{R}_t$}{
        $G(S) = \sum_{(j,k) \in \Pi_2(S)} D_{j,k}$ \label{algo-step:subset_grad}
    }
    \Return $\{G(S)\}_{S \in \mathcal{R}_t}$
}
\end{algorithm}

Here we present and analyze our proposed algorithm \ouralgo (\cref{algo:ouralgo}) for solving the D-optimal design problem \eqref{eq:optimal design}. Note that pseduocodes of some of the sub-routines used this algorithm are given \cref{algo:ouralgo} in the Appendix. \ouralgo is a fast randomized variant of the FW method which incorporates computationally efficient low-rank and sparse operations. Next, we describe various components of the algorithm and explain how they address the scaling concerns identified in \cref{sec:frank-wolfe method}.

\subsection{Randomized LMO and Cached Derivatives}

In \cref{sec:frank-wolfe method}, we saw that computation of the LMO has ${N \choose K}$ complexity due to the maximization step over $\cS$ \eqref{eq:fw_lmo}. Utilizing a randomized variant of FW method \cite{kerdreux2018frank}, we reduce this complexity to $R$ by restricting the maximization to an $R$ sized sub-collection $\cR_t$ chosen uniformly at random from $\cS$ (\cref{algo-step:sample,algo-step:lmo} of \cref{algo:ouralgo}):
\begin{align}
S_t \in &\argmax_{S \in \cR_t} G_t(S)\,,  \text{ where } \nonumber \\
\mathcal{R}_t &\sim \mathrm{Uniform}(\{\mathcal{R} \subseteq \mathcal{S} \mid |\mathcal{R}| = R \rceil\})\,.
\end{align}
This randomization of LMO also implies that we only need to compute $R$ partial gradients $G_t(S) = \frac{\partial g(\pi_t)}{\pi_t(S)}$ for the subsets $S \in \cR_t$. Next we see how to compute them efficiently.
Recall that $g(\pi) = \log\det(V_\pi)$ with $V_\pi = \sum_{S \in \cS} \pi(S) A_S A_S^\top$. Then using the fact that $\log\det(U) = U^{-1}$ for PSD matrix $U$ the partial derivative at any $S$ can be written as
\begin{align}
  \frac{\partial g(\pi)}{\partial (\pi(S))}
  = \langle A_S A_S^\top, V_{\pi}^{-1} \rangle 
  &= \hspace{-0.1in} \sum_{(j,k) \in \Pi_2(S)} \hspace{-0.1in}
  z_{j,k}^\top V_{\pi}^{-1} z_{j,k}
  \nonumber \\
  &\eqdef \sum_{(j,k) \in \Pi_2(S)}  D_{j,k}.
  \label{eq:gradient}
\end{align}
Note that all the partial derivatives are a sum of ${K \choose 2}$ terms chosen from the ${N \choose 2}$ terms in $\{D_{j,k}\}_{(j,k) \in \Pi_2([N])}$. Therefore, \partialgrad sub-routine first computes and caches $\{D_{j,k}\}_{(j,k)}$ (\cref{algo-step:pair_grad} in \cref{algo:ouralgo}) and then it combines them to produce the partial derivative corresponding to each $S \in \mathcal{R}_t$ (\cref{algo-step:subset_grad} in \cref{algo:ouralgo}). Assuming that $V_\pi^{-1}$ is known, this reduces our overall LMO complexity from $O({N \choose K} {K \choose 2} d^2)$ to $O({N \choose 2} d^2 + R {K \choose 2})$. Please note that these operations are further parallelized in our code.

\subsection{Line Search with Low-Rank and Sparse Updates}
\label{sec:line_search}

As noted earlier in \cref{sec:frank-wolfe method}, line search: $\max_\alpha g((1-\alpha) \cdot \pi_t + \alpha \cdot e_{S_t})$ is another expensive operation. \ouralgo solves it using the \goldensearch sub-routine (\cref{algo-step:golden-search} of \cref{algo:ouralgo}). 
\goldensearch (see \cref{algo:ouralgo_subroutines} in Appendix) solves it as 1-D unimodal maximization problem using the golden-section search method \cite{kiefer1953sequential} which reduces the search space by a factor of $\varphi$---the golden ratio---at each iteration. It computes the maximizer up to tolerance of $\alpha_{\mathrm{tol}}$ after $1 + \log_{\varphi}(\alpha_{\mathrm{tol}})$ objective computations. Earlier we also noted that naively computing the objective value has a complexity of $O(d^3)$. However, notice that for an $\alpha$ we can write the update to $V_\pi$ as
\begin{align}
(1-\alpha) V_\pi + \alpha A_S A_S^\top. \label{eq:v_pi_update}
\end{align}
Let $r = {K \choose 2} \ll d$. Then we see that the second term above is a low $r$ rank matrix since $A_S \in \mathbb{R}^{d \times r}$. Assuming access to $V_\pi^{-1}$, this allows us to derive the objective value as the $\log\det$ of an $r \times r$ matrix. The derivation is given in \cref{app:algo_additional_details}. This reduces the over overall line search complexity to $O((r^3 + rd^2) \log_{\varphi}(\alpha_{\mathrm{tol}}))$ from $O(d^3 \alpha_{\mathrm{tol}}^{-1})$.

Both the partial derivative and the objective value computations require the inverse of $V_{\pi_t}$. \ouralgo initialized it in the variable $V_0^{\mathrm{inv}}$ (\cref{algo-step:v_inv} of \cref{algo:ouralgo}) and updates it using \updateinverse (\cref{algo-step:update-inv} of \cref{algo:ouralgo}). Naively computing the inverse of a $d \times d$ matrix has $O(d^3)$ complexity. Again by noting that $V_\pi$ is updated with a low-rank matrix \eqref{eq:v_pi_update}, we can derive formula for $V_{t+1}^{\mathrm{inv}}$ using only an inverse of an $r \times r$ matrix (see \cref{app:algo_additional_details} for details). By implementing this formula \updateinverse (\cref{algo:ouralgo_subroutines}) reduces the inverse computation complexity from $O(d^3)$ to $O(r^3 + rd^2)$.
Finally, we address how \ouralgo maintains and stores the iterate $\pi_t$ with ${N \choose K}$ entries. When $N \gg K \gg 1$, this is expensive in terms of both compute and memory. Since FW updates are sparse (\cref{algo-step:update-iterate} of \cref{algo:ouralgo}), \ouralgo stores $\pi_t$ as a sparse matrix and updates it using sparse operations. More details are given in \cref{app:algo_additional_details}. This reduces the complexity of maintaining $\pi_t$ to its number of non-zero elements, which is at most $O(t)$ at iteration $t$ if \ouralgo is initialized with $\pi_0$ with a constant number of non-zero entries.

\subsection{Convergence Rate of Randomized FW Method}
\label{sec:convergence_rate}

In this section, we analyze the convergence rate of \ouralgo in terms of number of its iterations $T_{\mathrm{od}}$. From the previous sections, we see that \ouralgo is a randomized variant of the FW method. It is known that if $g$ were Lipschitz smooth, then \ouralgo would have converged to the maximizer of \eqref{eq:optimal design} \cite{kerdreux2018frank}. However, $g$ does not satisfy this condition due to the composition with $\log$. Instead we prove a more general result that randomized FW method (\cref{algo:rand_fw} in Appendix; generalization of \cref{algo:ouralgo}) solves problems of the form
\begin{align}
\min_{y \in \cY}  F(y) = \min_{y \in \cY} f(\mathcal{A}(y)), \label{eq:lhscb_problem}
\end{align}
where $\cY$ is the convex hull of a set $\cV$ of $\widetilde{N}$ vectors, $\cA: \cY \to \cK$ is a linear operator, $\cK$ is a regular cone, $f: \mathrm{int}(\cK) \to \mathbb{R}$ is a logarithmically-homogeneous self-concordant barrier (LHSCB) \cite{zhao2023analysis}.
\begin{theorem}
If the initial iterate $y_0$ maps to domain of $f$, i.e.~$\cA(y_0) \in \mathrm{int}(\cK)$, then Randomized FW method (\cref{algo:rand_fw}) which samples $R$ elements from $\cV$ for computing the randomized LMO achieves an $\varepsilon$ sub-optimal solution to the problem \eqref{eq:lhscb_problem} after $T_\mathrm{FW}=O(\max(1,\widetilde{N}/R)\varepsilon^{-1})$ iterations, i.e.~$F(y_{T_\mathrm{FW}}) - \min_{y \in \cY} F(y) \leq \varepsilon$.
\label{thm:rand_fw}
\end{theorem}
This novel result extends the convergence of randomized FW for Lipschitz smooth objectives to problems of the form \eqref{eq:lhscb_problem}. Next, as a corollary of \cref{thm:rand_fw}, we provide a convergence rate for \ouralgo.
\begin{corollary}
For full-rank $V_{\pi_0}$, $\gamma=0$, and small enough $\alpha_\mathrm{tol}$, \ouralgo (\cref{algo:ouralgo}) outputs an $\varepsilon$ sub-optimal solution to the D-optimal design problem \eqref{eq:optimal design} after $T_\mathrm{od}=O(\max(1, {N \choose K}/R)\varepsilon^{-1})$ iterations, i.e.~$g(\pi_{T_\mathrm{od}}) - \min_{\pi \in \Delta^{\mathcal{S}}} g(\pi) \leq \varepsilon$.
\label{cor:ouralgo}
\end{corollary}
We provide the proof for these statements in \cref{app:analysis}.
Note that \cref{cor:ouralgo} needs $V_{\pi_0}$  to be full-rank, otherwise the objective might not be well-defined. In practice this is prevented by regularizing the matrix as $V_{\pi_0} + \gamma \mathbf{I}_{d \times d}$ with a small $\gamma > 0$.
Finally, aggregating the complexities of the steps of \ouralgo, we see its per iteration complexity is $O({N \choose 2} d^2 + R {K \choose 2} + (r^3 + rd^2) \log_{\varphi}(\alpha_{\mathrm{tol}}) + T_{od})$, eliminating any $({N \choose K})$ and $O(d^3)$ dependence.

Note that the FW method (\cref{algo:fw}) implemented with all the tricks from \ouralgo except the randomization of LMO would have an iteration complexity of $O(\varepsilon^{-1})$ \cite{zhao2023analysis} and a per-iteration complexity of $O({N \choose 2} d^2 + {N \choose K} {K \choose 2} + (r^3 + rd^2) \log_{\varphi}(\alpha_{\mathrm{tol}}) + T_{od})$. Thus we see that the randomization of LMO leads to better per-iteration complexity while trading off the iteration complexity. This is because the proof of \cref{cor:ouralgo} pessimistically assumes that the objective function does not decrease at an iteration if the true LMO output is not preset in the sub-collection $\cR_t$. However, we do not observe this convergence rate degradation in practice (see \cref{sec:experiments}). Hence \ouralgo is a practical algorithm for solving $\eqref{eq:optimal design}$.

\section{Experiments}
\label{sec:experiments}

We conduct two experiments on three real-world datasets. The first experiment is with simulated feedback, for different sample sizes $T$ and $K$ in the $K$-way feedback. We call this \emph{synthetic feedback} setup. The second is with real feedback, where preference feedback present in the dataset is used. We refer to it as \textit{real feedback} setup. 

We experiment with (a) BEIR-COVID\footnote{https://huggingface.co/datasets/BeIR/trec-covid}, (b) Trec Deep Learning (TREC-DL)\footnote{https://microsoft.github.io/msmarco/TREC-Deep-Learning}, and (c) NECTAR\footnote{https://huggingface.co/datasets/berkeley-nest/Nectar} datasets. These are question-answering datasets, where the task is to rank answers or passages by relevance to the question. Each question in BEIR-COVID and TREC-DL has a hundred potential answers. Each question in NECTAR has seven potential answers. All the following experiments are conducted on 3.5 GHz 3rd generation Intel Xeon Scalable processors with 128 vCPUs and 1TB RAM.

\subsection{Synthetic Feedback Setup}
\label{ssec:synthetic}

\begin{figure*}[t]
    \centering
    \subfigure[BEIR-COVID, $K=2$]{
        \includegraphics[width=0.31\textwidth]{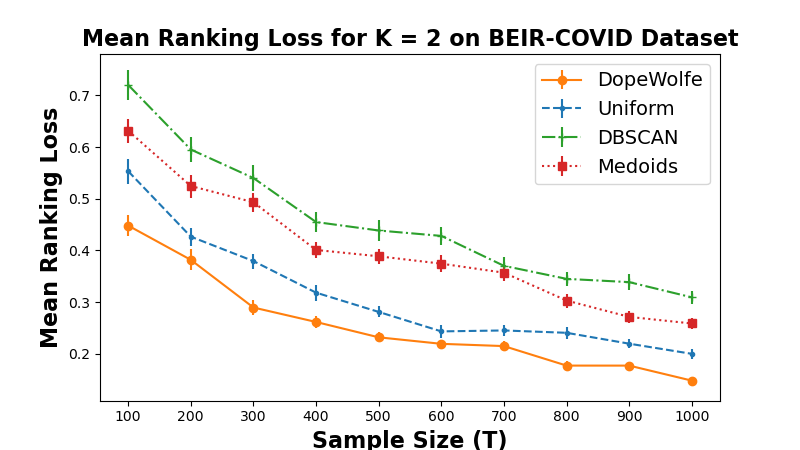}
        \label{fig:subfig2}
    }
    \subfigure[TREC-DL, $K=2$]{
        \includegraphics[width=0.31\textwidth]{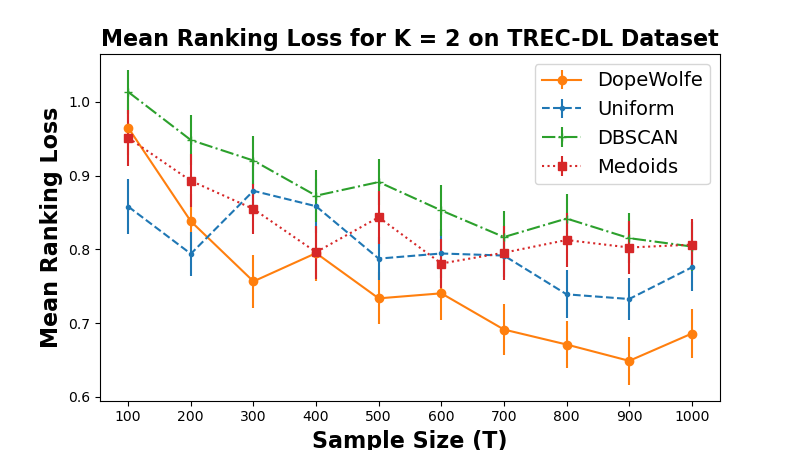}
        \label{fig:subfig1}
    }
    \subfigure[NECTAR, $K=2$]{
        \includegraphics[width=0.31\textwidth]{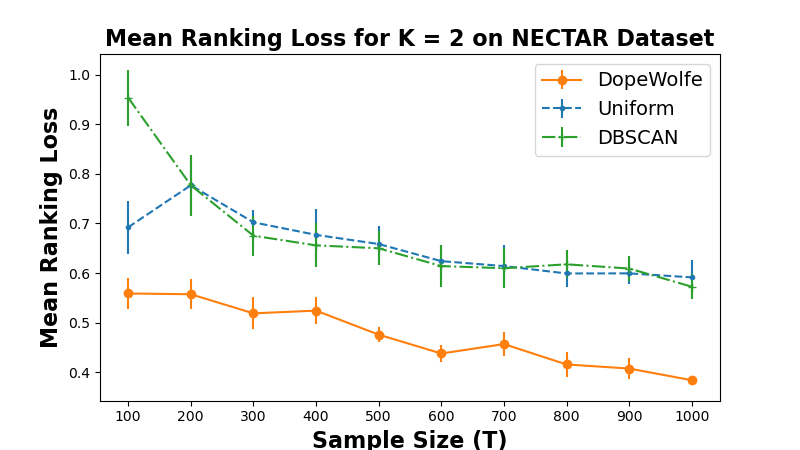}
        \label{fig:subfig3}
    }
    \subfigure[BEIR-COVID, $K=3$]{
        \includegraphics[width=0.31\textwidth]{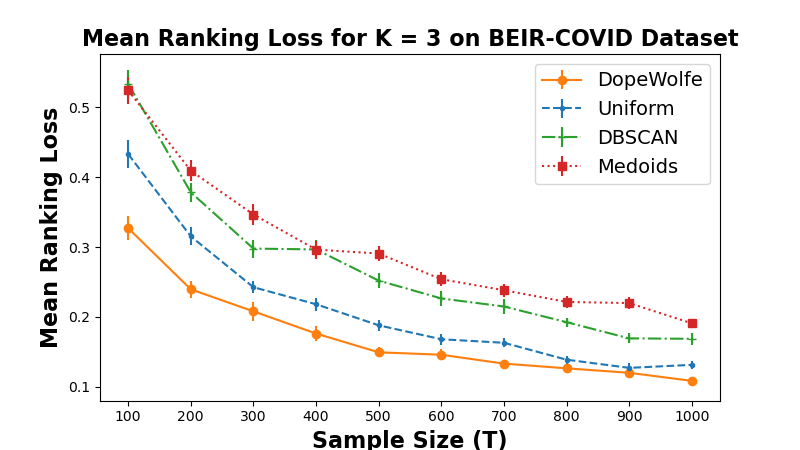}
        \label{fig:subfig4}
    }
    \subfigure[TREC-DL, $K=3$]{
        \includegraphics[width=0.31\textwidth]{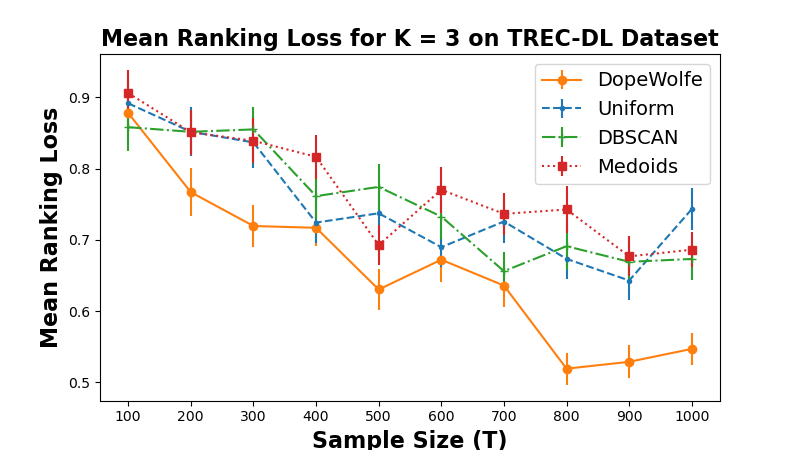}
        \label{fig:subfig5}
    }
    \subfigure[NECTAR, $K=3$]{
        \includegraphics[width=0.31\textwidth]{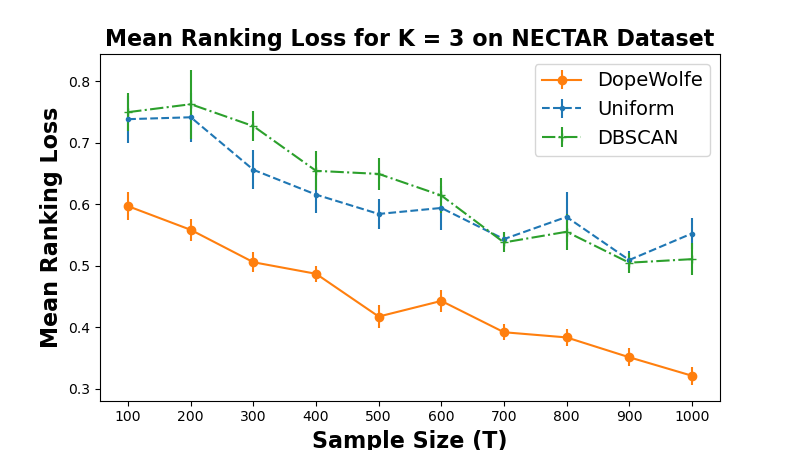}
        \label{fig:subfig6}
    }
    \caption{Mean KDtau metric on datasets with synthetic feedback. $k$-medoids fails for the NECTAR dataset due to excessive memory requirements (2 TB and 8 TB of RAM, respectively).}
    \label{fig:synthkdtau}
\end{figure*}

In this section, we provide empirical evidence demonstrating that \ouralgo, when applied to sampling $K$-sized subsets from a list, enables the learning of an improved ranking model across various values of $K$ and sample sizes $T$. We first compute $384$-dimensional dense BERT embeddings~\cite{reimers-2019-sentence-bert} for each question and answer, and then reduce the embedding dimensions to 10 by fitting UMAP~\cite{mcinnes2018umap} on the answers in a dataset. The same UMAP transformation is applied to questions to get 10-dimensional embedding of the questions. Let $q$ and $a$ be the projected embeddings of a question and an answer to it. Then, we consider the outer product $vec(qa^T)$ as the feature vector of the question-answer pair. Its length is $d = 100$. We choose a random $\theta^*\in \realset^{100}$ to generate feedback in \eqref{eq:ranking feedback}.

The output of \ouralgo is a  distribution over all $K$-subsets of the original set of $N$ items. To show the efficacy of \ouralgo, we compare it with the following baselines:

\begin{enumerate}[leftmargin=0.5cm]
    \item \textbf{Uniform:} This approach chooses $K$-subsets at random with equal probability.
    \item \textbf{DBSCAN:} This approach first applies the DBSCAN clustering algorithm~\cite{ester1996density} over features of all $K$-subsets. Then, it chooses any $K$-subset associated with a \textit{core sample} (centroid) at random with equal probability. We vary the  distance parameter in DBSCAN in $[10^{-5}, 1]$ and choose the one that yields the sparsest distribution. 
    \item \textbf{$m$-medoids:} We set $m$ to the size of the support of the distribution from \ouralgo. Then, we run the $m$-medoids algorithm\footnote{https://tinyurl.com/scikit-learn-m-medoids} over the features of all $K$-subsets. Finally, we select $K$-subsets associated with \textit{core samples} (centroids) at random with equal probabilities. Note that this approach has extra information in terms of $m$ which other methods do not have.  
\end{enumerate}

We consider 1 random question in BEIR-COVID and TREC-DL datasets. The number of 2-way subsets and 3-way subsets for one question with hundred answers are $100 \choose 2$ $=4950$ and $100 \choose 3$ $=161700$, respectively. We consider random 30,000 questions in NECTAR dataset. The number of 2-way subsets and 3-way subsets, where each question has seven answers, are $30000$$ 7 \choose 2$ $=630000$ and $30000$$ 7 \choose 3$ $=1050000$, respectively. Given a sample size $T$, we fit a PL ranking model. We repeat this experiment for 100 random runs and show the mean along with standard error of the ranking loss~\eqref{eq:ranking loss} for $K=2,3$ and $T=\{100,200,\ldots 1000\}$ in Figure~\ref{fig:synthkdtau}. NDCG is shown in Figure~\ref{fig:synthndcg} (Appendix~\ref{app:synthexperiments}). In addition, see Figure~\ref{fig:synthk4} (Appendix~\ref{app:synthexperiments}) for results on $K=4$. 

We observe that: (a) our approach achieves the best ranking performance for all the datasets, (b) DBSCAN is either equal or worse than Uniform (c) there is no clear winner between DBSCAN and $m$-medoids across different setups, (d) the ranking loss decreases as we increase $K$ for the same $T$, which is expected since larger $K$ allows for larger feedback, and (e) the ranking loss decreases as we increase $T$. $m$-medoids fails for the NECTAR dataset due to excessive memory requirements (2 TB and 8 TB of RAM, respectively). This shows that \ouralgo is indeed able to choose a few subsets from a large number of candidate subsets to learn a better ranking model than baselines. 

\subsection{Real Feedback Setup}
\label{ssec:real}

In this section, we consider the real feedback available in the TREC-DL and BEIR-COVID dataset. Particularly, these datasets provide the full ranking of all the top-100 answers for any question, with BEIR-COVID containing ties as well. We consider the same set up as Section \ref{ssec:synthetic} with a few changes. Since the feedback is real, we use more representative 1024 dimensional BGE-M3 embedding~\cite{bge-m3} to ensure that the features are rich and reasonable enough to capture the true ranking (realizable ranking problem). 
Additionally, since the underlying ranking is fixed and feedback is noiseless, we do not vary the sample size $T$. Instead, we fix $T$ and vary $K$ in $K$-way feedback. 

Note that large $N$ and moderate $K$ values make this an extremely large scale ranking problem. Particularly, when $N = 100$ and $K=10$, there are $\sim$$17$ Trillion possible subsets for which we can elicit feedback. Therefore, it is practically impossible to run the optimal design approach from~\cite{mukherjee24optimal} or vanilla Frank-Wolfe method even in our moderately compute heavy machine. We run \ouralgo, where we uniformly sub-sample a collection of $\min({100 \choose K}, 10^5)$ subsets and compute the approximate LMO at each iteration as discussed in Section~\ref{sec:scalable algorithm}. Please see Appendix~\ref{app:realexperiments} for more experimental details and choices. 

Similar to Section~\ref{ssec:synthetic},  we evaluate the feedback elicitation methods by assessing the downstream performance of the learned PL model. Specifically, we compare \ouralgo to Uniform sampling for feedback elicitation. DBSCAN and $m$-medoids are excluded from this comparison due to their inability to scale effectively for this problem size. In Figure~\ref{fig:real_feedback_exp}, we plot the ranking loss~\eqref{eq:ranking loss} for the real feedback experiments. See Figure~\ref{fig:real_feedback_ndcg} (Appendix~\ref{app:realexperiments}) for NDCG metric. We conclude that (a) \ouralgo always perform better than Uniform sampling, and (b) the performance of both the algorithms increase as $K$ increases. This is expected since larger $K$ allows for more information in a $K$-way feedback.

\begin{figure}[t]
    \centering
    \subfigure[Ranking loss on BEIR-COVID with real feedback]{
        \includegraphics[width=0.4\textwidth]{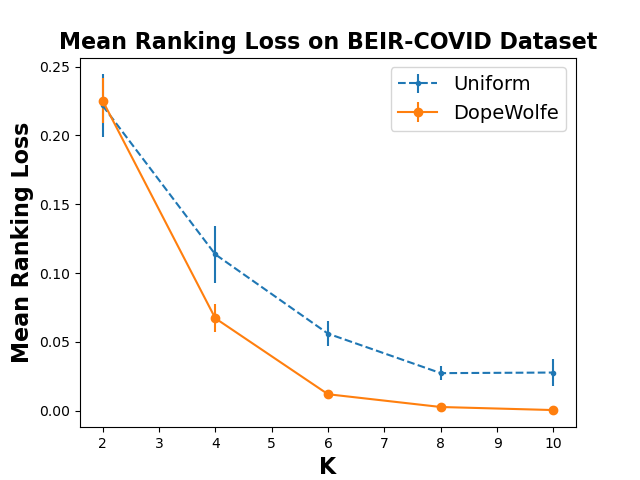}
        \label{fig:real_covid_subfig1}
    }
    \subfigure[Ranking loss on TREC-DL with real feedback]{
        \includegraphics[width=0.4\textwidth]{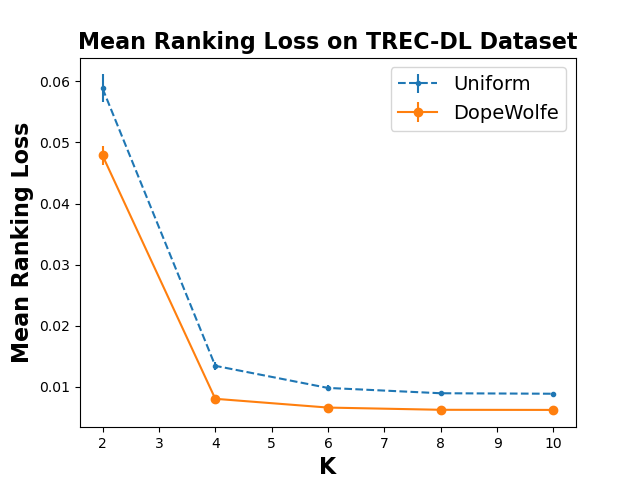}
        \label{fig:real_trecdl_subfig2}
    }
\caption{Ranking models learned through samples selected by \ouralgo achieve lesser \emph{ranking loss} than the ones learned through samples selected uniformly at random.}
    \label{fig:real_feedback_exp}    
\end{figure}

\begin{table}[t]
    \centering
    \begin{tabular}{|l|c|}
        \hline
        \textbf{Method} & \textbf{Time in Seconds} \\
        \hline
        \ouralgo at $1000$ iterations & 1646.70 \\
        DBSCAN $(\epsilon=10^{-5})$ & 37.52 \\
        DBSCAN $(\epsilon=10^{-4})$ & 37.46 \\
        DBSCAN $(\epsilon=10^{-3})$ & 39.63 \\
        DBSCAN $(\epsilon=10^{-2})$ & 69.28 \\
        DBSCAN $(\epsilon=10^{-1})$ & 1346.28 \\
        DBSCAN $(\epsilon=1)$ & 3397.21 \\
        \hline
    \end{tabular}
    \caption{Run time comparison: \ouralgo vs DBSCAN.}
    \label{tab:method_times}
\end{table}

\subsection{Practical Run Time Comparison}
\label{ssec:runtime}

We now compare the run time of \ouralgo with DBSCAN. Although the running times for \ouralgo and DBSCAN are not directly comparable due to different deciding parameters—\ouralgo relies on the number of Frank-Wolfe iterations, while DBSCAN depends on the distance threshold ($\epsilon$). We chose DBSCAN as a baseline clustering algorithm, because it does not require an input parameter specifying distribution sparsity (e.g., $m$-medoids requiring $m$), and it provides cluster centers as actual data points. 

Nonetheless, in \cref{tab:method_times}, we present run-time comparisons for \ouralgo with 1000 Frank-Wolfe iterations and DBSCAN with distance thresholds varying between $[10^{-5}, \ldots, 1]$, as explained in Section~\ref{ssec:synthetic} for the BEIR-COVID dataset with $K=3$, corresponding to 161,700 possible subsets. We observe that DBSCAN takes twice the time of our proposed method for $\epsilon=1$. The total run-time for the DBSCAN-based baseline is 4927.38 seconds; whereas ours is only 1646.70 seconds. Although we used 1000 iterations in all our experiments, convergence is typically achieved in around 100 iterations, yielding similar performance, which only takes 164.67 seconds. This demonstrates that our proposed method not only achieves superior performance but also is computationally efficient.

\section{Conclusions}
\label{sec:conclusions}

We studied learning of human preferences for $N$ items from $K$-way ranking feedback under a limited feedback budget. We developed a computationally and statistically efficient optimal design algorithm for selecting a small collection of $K$-subsets, each for eliciting $K$-way ranking feedback. The policy for collecting the feedback is a solution to the D-optimal design. We solve it using the FW method combined with a novel approximate linear maximization oracle over the simplex of all subsets of size $K$. This allows us to avoid the $O({N \choose K})$ time complexity of a naive solution. We empirically evaluate our approach on both synthetic and real-world datasets, with relatively large $N$ and $K$.

Our work can be extended in multiple directions. First, we only considered one particular model of human feedback, $K$-way questions, popularized by RLHF \citep{ouyang22training} and DPO \citep{rafailov23direct}. Second, the $O(d^2)$ dependence in \cref{prop:prediction error} is suboptimal and can be reduced to $O(d)$, by avoiding the Cauchy-Schwarz inequality in \eqref{eq:prediction error}. Finally, we may experiment with higher dimensional embeddings. This would allow us to efficiently perform LLM alignement \cite{sun2023chatgpt}.

\bibliographystyle{plainnat}
\bibliography{Brano,Subho,ranking_design}

\begin{thebibliography}{60}
\providecommand{\natexlab}[1]{#1}
\providecommand{\url}[1]{\texttt{#1}}
\expandafter\ifx\csname urlstyle\endcsname\relax
  \providecommand{\doi}[1]{doi: #1}\else
  \providecommand{\doi}{doi: \begingroup \urlstyle{rm}\Url}\fi

\bibitem[Abdi and Williams(2010)]{abdi2010principal}
Herv{\'e} Abdi and Lynne~J Williams.
\newblock Principal component analysis.
\newblock \emph{Wiley interdisciplinary reviews: computational statistics}, 2\penalty0 (4):\penalty0 433--459, 2010.

\bibitem[Agichtein et~al.(2006)Agichtein, Brill, Dumais, and Ragno]{agichtein06learning}
Eugene Agichtein, Eric Brill, Susan Dumais, and Robert Ragno.
\newblock Learning user interaction models for predicting web search result preferences.
\newblock In \emph{Proceedings of the 29th Annual International ACM SIGIR Conference}, pages 3--10, 2006.

\bibitem[Ahipa{\c{s}}ao{\u{g}}lu(2015)]{ahipacsaouglu2015first}
Selin~Damla Ahipa{\c{s}}ao{\u{g}}lu.
\newblock A first-order algorithm for the a-optimal experimental design problem: a mathematical programming approach.
\newblock \emph{Statistics and Computing}, 25:\penalty0 1113--1127, 2015.

\bibitem[Audibert et~al.(2010)Audibert, Bubeck, and Munos]{audibert10best}
Jean-Yves Audibert, Sebastien Bubeck, and Remi Munos.
\newblock Best arm identification in multi-armed bandits.
\newblock In \emph{Proceedings of the 23rd Annual Conference on Learning Theory}, pages 41--53, 2010.

\bibitem[Azizi et~al.(2022)Azizi, Kveton, and Ghavamzadeh]{azizi22fixedbudget}
Mohammad~Javad Azizi, Branislav Kveton, and Mohammad Ghavamzadeh.
\newblock Fixed-budget best-arm identification in structured bandits.
\newblock In \emph{Proceedings of the 31st International Joint Conference on Artificial Intelligence}, 2022.

\bibitem[Barzilai and Borwein(1988)]{barzilai1988two}
Jonathan Barzilai and Jonathan~M Borwein.
\newblock Two-point step size gradient methods.
\newblock \emph{IMA journal of numerical analysis}, 8\penalty0 (1):\penalty0 141--148, 1988.

\bibitem[Boyd and Vandenberghe(2004)]{boyd04convex}
Stephen Boyd and Lieven Vandenberghe.
\newblock \emph{Convex Optimization}.
\newblock Cambridge University Press, Cambridge, United Kingdom, 2004.

\bibitem[Bradley and Terry(1952)]{bradley52rank}
Ralph~Allan Bradley and Milton Terry.
\newblock Rank analysis of incomplete block designs: I. the method of paired comparisons.
\newblock \emph{Biometrika}, 39\penalty0 (3-4):\penalty0 324--345, 1952.

\bibitem[Bubeck et~al.(2009)Bubeck, Munos, and Stoltz]{bubeck09pure}
Sebastien Bubeck, Remi Munos, and Gilles Stoltz.
\newblock Pure exploration in multi-armed bandits problems.
\newblock In \emph{Proceedings of the 20th International Conference on Algorithmic Learning Theory}, pages 23--37, 2009.

\bibitem[Casper et~al.(2023)Casper, Davies, Shi, Gilbert, Scheurer, Rando, Freedman, Korbak, Lindner, Freire, et~al.]{casper2023open}
Stephen Casper, Xander Davies, Claudia Shi, Thomas~Krendl Gilbert, J{\'e}r{\'e}my Scheurer, Javier Rando, Rachel Freedman, Tomasz Korbak, David Lindner, Pedro Freire, et~al.
\newblock Open problems and fundamental limitations of reinforcement learning from human feedback.
\newblock \emph{arXiv preprint arXiv:2307.15217}, 2023.

\bibitem[Chen et~al.(2024)Chen, Xiao, Zhang, Luo, Lian, and Liu]{bge-m3}
Jianlv Chen, Shitao Xiao, Peitian Zhang, Kun Luo, Defu Lian, and Zheng Liu.
\newblock Bge m3-embedding: Multi-lingual, multi-functionality, multi-granularity text embeddings through self-knowledge distillation, 2024.

\bibitem[Das et~al.(2024)Das, Chakraborty, Pacchiano, and Chowdhury]{das24active}
Nirjhar Das, Souradip Chakraborty, Aldo Pacchiano, and Sayak~Ray Chowdhury.
\newblock Active preference optimization for sample efficient {RLHF}.
\newblock \emph{CoRR}, abs/2402.10500, 2024.
\newblock URL \url{https://arxiv.org/abs/2402.10500}.

\bibitem[Diamond and Boyd(2016)]{diamond2016cvxpy}
Steven Diamond and Stephen Boyd.
\newblock {CVXPY}: {A} {P}ython-embedded modeling language for convex optimization.
\newblock \emph{Journal of Machine Learning Research}, 17\penalty0 (83):\penalty0 1--5, 2016.

\bibitem[Dunn and Harshbarger(1978)]{dunn1978conditional}
Joseph~C Dunn and S~Harshbarger.
\newblock Conditional gradient algorithms with open loop step size rules.
\newblock \emph{Journal of Mathematical Analysis and Applications}, 62\penalty0 (2):\penalty0 432--444, 1978.

\bibitem[Ester et~al.(1996)Ester, Kriegel, Sander, Xu, et~al.]{ester1996density}
Martin Ester, Hans-Peter Kriegel, J{\"o}rg Sander, Xiaowei Xu, et~al.
\newblock A density-based algorithm for discovering clusters in large spatial databases with noise.
\newblock In \emph{KDD}, volume~96, pages 226--231, 1996.

\bibitem[Fedorov(2013)]{fedorov2013theory}
Valerii~Vadimovich Fedorov.
\newblock \emph{Theory of optimal experiments}.
\newblock Elsevier, 2013.

\bibitem[Frank et~al.(1956)Frank, Wolfe, et~al.]{frank1956algorithm}
Marguerite Frank, Philip Wolfe, et~al.
\newblock An algorithm for quadratic programming.
\newblock \emph{Naval research logistics quarterly}, 3\penalty0 (1-2):\penalty0 95--110, 1956.

\bibitem[Guttman(1946)]{guttman1946enlargement}
Louis Guttman.
\newblock Enlargement methods for computing the inverse matrix.
\newblock \emph{The annals of mathematical statistics}, pages 336--343, 1946.

\bibitem[Hendrych et~al.(2023)Hendrych, Besan{\c{c}}on, and Pokutta]{hendrych2023solving}
Deborah Hendrych, Mathieu Besan{\c{c}}on, and Sebastian Pokutta.
\newblock Solving the optimal experiment design problem with mixed-integer convex methods.
\newblock \emph{arXiv preprint arXiv:2312.11200}, 2023.

\bibitem[Hiranandani et~al.(2019)Hiranandani, Singh, Gupta, Burhanuddin, Wen, and Kveton]{hiranandani19cascading}
Gaurush Hiranandani, Harvineet Singh, Prakhar Gupta, Iftikhar~Ahamath Burhanuddin, Zheng Wen, and Branislav Kveton.
\newblock Cascading linear submodular bandits: Accounting for position bias and diversity in online learning to rank.
\newblock In \emph{Proceedings of the 35th Conference on Uncertainty in Artificial Intelligence}, 2019.

\bibitem[Hofmann et~al.(2013)Hofmann, Whiteson, and de~Rijke]{hofmann13fidelity}
Katja Hofmann, Shimon Whiteson, and Maarten de~Rijke.
\newblock Fidelity, soundness, and efficiency of interleaved comparison methods.
\newblock \emph{ACM Transactions on Information Systems}, 31\penalty0 (4):\penalty0 1--43, 2013.

\bibitem[Hofmann et~al.(2016)Hofmann, Li, and Radlinski]{hofmann16online}
Katja Hofmann, Lihong Li, and Filip Radlinski.
\newblock Online evaluation for information retrieval.
\newblock \emph{Foundations and Trends in Information Retrieval}, 2016.

\bibitem[Jaggi(2013)]{jaggi2013revisiting}
Martin Jaggi.
\newblock Revisiting frank-wolfe: Projection-free sparse convex optimization.
\newblock In \emph{International conference on machine learning}, pages 427--435. PMLR, 2013.

\bibitem[Jin et~al.(2021)Jin, Yang, and Wang]{jin2021pessimism}
Ying Jin, Zhuoran Yang, and Zhaoran Wang.
\newblock Is pessimism provably efficient for offline rl?
\newblock In \emph{International Conference on Machine Learning}, pages 5084--5096. PMLR, 2021.

\bibitem[Kang et~al.(2023)Kang, Shi, Liu, He, and Wang]{kang2023beyond}
Yachen Kang, Diyuan Shi, Jinxin Liu, Li~He, and Donglin Wang.
\newblock Beyond reward: Offline preference-guided policy optimization.
\newblock \emph{arXiv preprint arXiv:2305.16217}, 2023.

\bibitem[Kendall(1948)]{kendall1948rank}
Maurice~George Kendall.
\newblock \emph{Rank correlation methods.}
\newblock Griffin, 1948.

\bibitem[Kerdreux et~al.(2018)Kerdreux, Pedregosa, and d’Aspremont]{kerdreux2018frank}
Thomas Kerdreux, Fabian Pedregosa, and Alexandre d’Aspremont.
\newblock Frank-wolfe with subsampling oracle.
\newblock In \emph{International Conference on Machine Learning}, pages 2591--2600. PMLR, 2018.

\bibitem[Khachiyan(1996)]{khachiyan1996rounding}
Leonid~G Khachiyan.
\newblock Rounding of polytopes in the real number model of computation.
\newblock \emph{Mathematics of Operations Research}, 21\penalty0 (2):\penalty0 307--320, 1996.

\bibitem[Khetan and Oh(2016)]{khetan2016data}
Ashish Khetan and Sewoong Oh.
\newblock Data-driven rank breaking for efficient rank aggregation.
\newblock In \emph{International Conference on Machine Learning}, pages 89--98. PMLR, 2016.

\bibitem[Kiefer(1953)]{kiefer1953sequential}
Jack Kiefer.
\newblock Sequential minimax search for a maximum.
\newblock \emph{Proceedings of the American mathematical society}, 4\penalty0 (3):\penalty0 502--506, 1953.

\bibitem[Kiefer and Wolfowitz(1960)]{kiefer1960equivalence}
Jack Kiefer and Jacob Wolfowitz.
\newblock The equivalence of two extremum problems.
\newblock \emph{Canadian Journal of Mathematics}, 12:\penalty0 363--366, 1960.

\bibitem[Kveton et~al.(2015)Kveton, Szepesvari, Wen, and Ashkan]{kveton15cascading}
Branislav Kveton, Csaba Szepesvari, Zheng Wen, and Azin Ashkan.
\newblock Cascading bandits: Learning to rank in the cascade model.
\newblock In \emph{Proceedings of the 32nd International Conference on Machine Learning}, 2015.

\bibitem[Lagree et~al.(2016)Lagree, Vernade, and Cappe]{lagree16multipleplay}
Paul Lagree, Claire Vernade, and Olivier Cappe.
\newblock Multiple-play bandits in the position-based model.
\newblock In \emph{Advances in Neural Information Processing Systems 29}, pages 1597--1605, 2016.

\bibitem[Lattimore and Szepesvari(2019)]{lattimore19bandit}
Tor Lattimore and Csaba Szepesvari.
\newblock \emph{Bandit Algorithms}.
\newblock Cambridge University Press, 2019.

\bibitem[Li et~al.(2016)Li, Wang, Zhang, and Chen]{li16contextual}
Shuai Li, Baoxiang Wang, Shengyu Zhang, and Wei Chen.
\newblock Contextual combinatorial cascading bandits.
\newblock In \emph{Proceedings of the 33rd International Conference on Machine Learning}, pages 1245--1253, 2016.

\bibitem[Luce(2005)]{luce05individual}
Robert~Duncan Luce.
\newblock \emph{Individual Choice Behavior: A Theoretical Analysis}.
\newblock Dover Publications, 2005.

\bibitem[Manning et~al.(2008)Manning, Raghavan, and Schutze]{manning08introduction}
Christopher Manning, Prabhakar Raghavan, and Hinrich Schutze.
\newblock \emph{Introduction to Information Retrieval}.
\newblock Cambridge University Press, 2008.

\bibitem[McInnes et~al.(2018)McInnes, Healy, and Melville]{mcinnes2018umap}
Leland McInnes, John Healy, and James Melville.
\newblock Umap: Uniform manifold approximation and projection for dimension reduction.
\newblock \emph{arXiv preprint arXiv:1802.03426}, 2018.

\bibitem[Mehta et~al.(2023)Mehta, Das, Neopane, Dai, Bogunovic, Schneider, and Neiswanger]{mehta23sample}
Viraj Mehta, Vikramjeet Das, Ojash Neopane, Yijia Dai, Ilija Bogunovic, Jeff Schneider, and Willie Neiswanger.
\newblock Sample efficient reinforcement learning from human feedback via active exploration.
\newblock \emph{CoRR}, abs/2312.00267, 2023.
\newblock URL \url{https://arxiv.org/abs/2312.00267}.

\bibitem[Mukherjee et~al.(2024)Mukherjee, Lalitha, Kalantari, Deshmukh, Liu, Ma, and Kveton]{mukherjee24optimal}
Subhojyoti Mukherjee, Anusha Lalitha, Kousha Kalantari, Aniket Deshmukh, Ge~Liu, Yifei Ma, and Branislav Kveton.
\newblock Optimal design for human feedback.
\newblock \emph{CoRR}, abs/2404.13895, 2024.
\newblock URL \url{https://arxiv.org/abs/2404.13895}.

\bibitem[Negahban et~al.(2018)Negahban, Oh, Thekumparampil, and Xu]{negahban2018learning}
Sahand Negahban, Sewoong Oh, Kiran~K Thekumparampil, and Jiaming Xu.
\newblock Learning from comparisons and choices.
\newblock \emph{The Journal of Machine Learning Research}, 19\penalty0 (1):\penalty0 1478--1572, 2018.

\bibitem[Nesterov(2014)]{Nesterov}
Yurii Nesterov.
\newblock \emph{Introductory Lectures on Convex Optimization: A Basic Course}.
\newblock Springer Publishing Company, Incorporated, 1 edition, 2014.

\bibitem[Ouyang et~al.(2022)Ouyang, Wu, Jiang, Almeida, Wainwright, Mishkin, Zhang, Agarwal, Slama, Ray, Schulman, Hilton, Kelton, Miller, Simens, Askell, Welinder, Christiano, Leike, and Lowe]{ouyang22training}
Long Ouyang, Jeffrey Wu, Xu~Jiang, Diogo Almeida, Carroll Wainwright, Pamela Mishkin, Chong Zhang, Sandhini Agarwal, Katarina Slama, Alex Ray, John Schulman, Jacob Hilton, Fraser Kelton, Luke Miller, Maddie Simens, Amanda Askell, Peter Welinder, Paul Christiano, Jan Leike, and Ryan Lowe.
\newblock Training language models to follow instructions with human feedback.
\newblock In \emph{Advances in Neural Information Processing Systems 35}, 2022.

\bibitem[Pan et~al.(2022)Pan, Tsang, Chen, Niu, and Sugiyama]{pan2022fast}
Yuangang Pan, Ivor~W Tsang, Weijie Chen, Gang Niu, and Masashi Sugiyama.
\newblock Fast and robust rank aggregation against model misspecification.
\newblock \emph{Journal of Machine Learning Research}, 23\penalty0 (23):\penalty0 1--35, 2022.

\bibitem[Plackett(1975)]{plackett75analysis}
Robin~Lewis Plackett.
\newblock The analysis of permutations.
\newblock \emph{Journal of the Royal Statistical Society: Series C (Applied Statistics)}, 24\penalty0 (2):\penalty0 193--202, 1975.

\bibitem[Pozrikidis(2014)]{pozrikidis2014introduction}
Constantine Pozrikidis.
\newblock \emph{An introduction to grids, graphs, and networks}.
\newblock Oxford University Press, USA, 2014.

\bibitem[Pukelsheim(1993)]{pukelsheim93optimal}
Friedrich Pukelsheim.
\newblock \emph{Optimal Design of Experiments}.
\newblock John Wiley \& Sons, 1993.

\bibitem[Radlinski et~al.(2008)Radlinski, Kleinberg, and Joachims]{radlinski08learning}
Filip Radlinski, Robert Kleinberg, and Thorsten Joachims.
\newblock Learning diverse rankings with multi-armed bandits.
\newblock In \emph{Proceedings of the 25th International Conference on Machine Learning}, pages 784--791, 2008.

\bibitem[Rafailov et~al.(2023)Rafailov, Sharma, Mitchell, Manning, Ermon, and Finn]{rafailov23direct}
Rafael Rafailov, Archit Sharma, Eric Mitchell, Christopher Manning, Stefano Ermon, and Chelsea Finn.
\newblock Direct preference optimization: Your language model is secretly a reward model.
\newblock In \emph{Advances in Neural Information Processing Systems 36}, 2023.

\bibitem[Reimers and Gurevych(2019)]{reimers-2019-sentence-bert}
Nils Reimers and Iryna Gurevych.
\newblock Sentence-bert: Sentence embeddings using siamese bert-networks.
\newblock In \emph{Proceedings of the 2019 Conference on Empirical Methods in Natural Language Processing}. Association for Computational Linguistics, 11 2019.
\newblock URL \url{http://arxiv.org/abs/1908.10084}.

\bibitem[Sun et~al.(2023)Sun, Yan, Ma, Ren, Yin, and Ren]{sun2023chatgpt}
Weiwei Sun, Lingyong Yan, Xinyu Ma, Pengjie Ren, Dawei Yin, and Zhaochun Ren.
\newblock Is chatgpt good at search? investigating large language models as re-ranking agent.
\newblock \emph{arXiv preprint arXiv:2304.09542}, 2023.

\bibitem[Tennenholtz et~al.(2024)Tennenholtz, Chow, Hsu, Shani, Liang, and Boutilier]{tennenholtz2024embedding}
Guy Tennenholtz, Yinlam Chow, Chih-Wei Hsu, Lior Shani, Ethan Liang, and Craig Boutilier.
\newblock Embedding-aligned language models.
\newblock \emph{arXiv preprint arXiv:2406.00024}, 2024.

\bibitem[Tversky and Kahneman(1974)]{tversky74judgment}
Amos Tversky and Daniel Kahneman.
\newblock Judgment under uncertainty: Heuristics and biases.
\newblock \emph{Science}, 185\penalty0 (4157):\penalty0 1124--1131, 1974.

\bibitem[Woodbury(1950)]{woodbury1950inverting}
Max~A Woodbury.
\newblock \emph{Inverting modified matrices}.
\newblock Department of Statistics, Princeton University, 1950.

\bibitem[Xu et~al.(2022)Xu, Wang, Zou, and Liang]{xu2022provably}
Tengyu Xu, Yue Wang, Shaofeng Zou, and Yingbin Liang.
\newblock Provably efficient offline reinforcement learning with trajectory-wise reward.
\newblock \emph{arXiv preprint arXiv:2206.06426}, 2022.

\bibitem[Yang and Tan(2022)]{yang22minimax}
Junwen Yang and Vincent Tan.
\newblock Minimax optimal fixed-budget best arm identification in linear bandits.
\newblock In \emph{Advances in Neural Information Processing Systems 35}, 2022.

\bibitem[Zanette(2023)]{zanette2023realizability}
Andrea Zanette.
\newblock When is realizability sufficient for off-policy reinforcement learning?
\newblock In \emph{International Conference on Machine Learning}, pages 40637--40668. PMLR, 2023.

\bibitem[Zhao and Freund(2023)]{zhao2023analysis}
Renbo Zhao and Robert~M Freund.
\newblock Analysis of the frank--wolfe method for convex composite optimization involving a logarithmically-homogeneous barrier.
\newblock \emph{Mathematical programming}, 199\penalty0 (1):\penalty0 123--163, 2023.

\bibitem[Zhu et~al.(2023)Zhu, Jiao, and Jordan]{zhu23principled}
Banghua Zhu, Jiantao Jiao, and Michael Jordan.
\newblock Principled reinforcement learning with human feedback from pairwise or {$K$}-wise comparisons.
\newblock \emph{CoRR}, abs/2301.11270, 2023.
\newblock URL \url{https://arxiv.org/abs/2301.11270}.

\bibitem[Zong et~al.(2016)Zong, Ni, Sung, Ke, Wen, and Kveton]{zong16cascading}
Shi Zong, Hao Ni, Kenny Sung, Nan~Rosemary Ke, Zheng Wen, and Branislav Kveton.
\newblock Cascading bandits for large-scale recommendation problems.
\newblock In \emph{Proceedings of the 32nd Conference on Uncertainty in Artificial Intelligence}, 2016.

\end{thebibliography}

\clearpage

\clearpage
\onecolumn
\appendix

\section{Related Work}
\label{sec:related work}

Three recent papers studied a similar setting. \citet{mehta23sample} and \citet{das24active} learn to rank pairs of items from pairwise feedback. \citet{mukherjee24optimal} learn to rank lists of $K$ items from $K$-way feedback in \eqref{eq:ranking feedback}. We propose a general framework for learning to rank $N$ items from $K$-way feedback, where $K \geq 2$ and $N \geq K$. Therefore, our setting is more general than in prior works. Our setting is also related to other bandit settings as follows. Due to the sample budget, it is reminiscent of fixed-budget \emph{best arm identification (BAI)} \citep{bubeck09pure,audibert10best,azizi22fixedbudget,yang22minimax}. The main difference is that we do not want to identify the best arm. We want to estimate the mean rewards of $N$ items to sort them. Online learning to rank has also been studied extensively \citep{radlinski08learning,kveton15cascading,zong16cascading,li16contextual,lagree16multipleplay,hiranandani19cascading}. We do not minimize cumulative regret or try to identify the best arm.

We discuss related work on Frank-Wolfe methods next. \citet{hendrych2023solving} optimize an experimental design by the Frank-Wolfe algorithm within a mixed-integer convex optimization framework. \citet{ahipacsaouglu2015first} presents two first-order Frank-Wolfe algorithms with rigorous time-complexity analyses for the A-optimal experimental design. \citet{zhao2023analysis} introduce a generalized Frank-Wolfe method for composite optimization with a logarithmically-homogeneous self-concordant barrier. None of these works focus on ranking problems. Our randomized Frank-Wolfe algorithm addresses a specific computational problem of solving a D-optimal design with $O({N \choose K})$ variables in learning to rank. While other works have learning to rank with ranking feedback, motivated by learning preference models in RLHF \cite{rafailov23direct, kang2023beyond, casper2023open}, our work is unique in applying optimal designs to collect human feedback for learning to rank. \citet{tennenholtz2024embedding} show the benefits of using an optimal design of a state-dependent action set to improve an embedding-aligned guided language agent's efficiency. Additionally, \citet{khetan2016data} and \citet{pan2022fast} provide optimal rank-breaking estimators for efficient rank aggregation under different user homogeneity assumptions. Our work stands out by integrating optimal design with randomized Frank-Wolfe method in the context of ranking a large list of $N$ items through $K$-way ranking feedback.

In offline RL, the agent directly observes the past history of interactions. Note that these actions can be suboptimal and there can be issues of data coverage and distribution shifts. Therefore in recent years pessimism under offline RL has gained traction \cite{jin2021pessimism, xu2022provably, zanette2023realizability}. In contrast to these works, we study offline K-wise preference ranking under PL and BTL models for pure exploration setting. We do not use any pessimism but use optimal design \cite{pukelsheim93optimal, fedorov2013theory} to ensure diversity among the data collected. If the action set is infinite, then approximately optimal designs can sometimes
be found efficiently \cite{lattimore19bandit}. \cite{sun2023chatgpt} investigated generative LLMs, such as ChatGPT and GPT-4, for relevance ranking in IR and found that properly instructed LLMs can deliver competitive, even superior results to state-of-the-art supervised methods on popular IR benchmarks. We note that our approach can further be used in such settings to enhance the performance.

\section{Theoretical convergence rate analyses}
\label{app:analysis}

This section provides the theoretical convergence rate analysis of \ouralgo (\cref{algo:ouralgo}). We do this by first proving a more general result for the broad class of LHSCB problems \eqref{eq:lhscb_problem}. Then we will provide the convergence rate of \ouralgo as a corollary of this result.

\subsection{Generic Randomized FW method for solving the LHSCB problem} \label{app:generic_analysis}

In this section we analyze the convergence of the generic randomize FW method (\cref{algo:rand_fw}) for minimizing composition of $\theta$-LHSCB functions and affine transformations. More precisely, we recall the studied optimization problem
\begin{align}
\min_{y \in \cY}  F(y) = \min_{y \in \cY} f(\mathcal{A}(y)), \label{eq:lhscb_problem_appendix}
\end{align}
where $\cY$ is the convex hull of a set $\cV$ of $\widetilde{N}$ vectors ($\cY = \mathrm{conv}(\cV)$), $\cA: \cY \to \cK$ is a linear operator, $\cK$ is a cone, and $f: \mathrm{int}(\cK) \to \mathbb{R}$ is a \emph{$\theta$-logarithmically-homogeneous self-concordant barrier (LHSCB)} \cite{zhao2023analysis}.
Next we define the LHSCB function.
\begin{definition}
We say that a function $f$ is \emph{$\theta$-logarithmically homogenous self-concordant barrier (LHSCB)} if
\begin{enumerate}
\item $f: \cK \to \mathbb{R}$ is a convex mapping with a regular cone $\cK \subsetneq \mathbb{R}^m$ (closed, convex, pointed and with non-empty interior $\mathrm{int}(\cK)$) as its domain,
\item $f$ is $\theta$-logarithmically homogenous, i.e.~$f(t u) = f(u) - \theta \log t$, for all $u \in \mathrm{int}(\cK)$ and $t > 0$ for some $\theta \geq 1$,
\item $f$ is self-concordant, i.e.~$|D^3 f(u) [v,v,v]|^2 \leq 4 \langle v, \nabla^2 f(u) v\rangle^3 $ for all $u \in \mathrm{int}(\cK)$ and $v \in \mathbb{R}^m$, where $D^3 f(u) [v,v,v]$ is the third order derivative of $f$ at $u$ in the direction $v$ and $\nabla^2 f(u)$ is the Hessian of $f$ at $u$, and
\item $f$ is a barrier, i.e.~$f(u_k) \to +\infty$ for any $(u_k)_{k \geq 0} \subseteq \mathrm{int}(\cK)$ and $u_k \to u \in \mathrm{bd}(\cK)$, where $\mathrm{bd}(\cK)$ is the boundary of $\cK$.
\end{enumerate}
\label{assume:lhscb}
\end{definition}
From the above assumption it is clear that problem \eqref{eq:lhscb_problem_appendix} is a convex problem.

\begin{algorithm}[htbp]
\caption{Randomized Frank-Wolfe Method for LHSCB Problem \eqref{eq:lhscb_problem_appendix}}
\label{algo:rand_fw}
\KwIn{$y_0$ such that $\cA(y_0) \in \mathrm{int}(\cK)$, sample size $R$}
\For{$t = 0, 1, \ldots, T_\mathrm{FW}$}{
    Sample $R$ ($\leq \widetilde{N}$) subsets uniformly at random from $\cV$: $\mathcal{R}_t \sim \mathrm{Uniform}(\{\mathcal{R} \subseteq \mathcal{S} \mid |\mathcal{R}| = R \rceil\})$  \\
    Compute $v_t = \argmin\limits_{v \in \cR_t} \langle \nabla F(y_t), v \rangle$ \tcp{Randomizes LMO}
    Compute $\alpha_t = \argmin\limits_{\alpha \in [0 , 1]} F(x_t + \alpha(v_t - x_t))$ \tcp{Line Search}
    Update $x_{t+1} = x_t + \alpha_t(v_t - x_t)$
}
\end{algorithm}
In \cref{algo:rand_fw}, we provide the generic randomize FW method for solving problem \eqref{eq:lhscb_problem_appendix}.
Before providing its convergence rate we state a known property of LHSCB functions.
It is know that LHSCB functions satisfy the following approximate second-order upper bound. We define, for a given positive semi-definite matrix $M$, the weighted norm $\| u\|_{M} = \sqrt{u^{T}Mu}$.
\begin{proposition}[\cite{zhao2023analysis,Nesterov}]
Let $f: \cK \to \mathbb{R}$ be LHSCB. Then,
\begin{align}
f(u) \leq f(v) + \langle \nabla f(u), v - u \rangle + \omega(\|v - u\|_{\nabla^2 f(u)})\,, \;\; \forall u \in \mathrm{int}(\cK), v \in \cK
\end{align}
where
\begin{equation}
\omega(a) = 
\begin{cases}
-a - \ln(1-a), & a < 1 \\
+\infty,& 1 \leq a.
\end{cases}
\end{equation}
\label{prop:lhscb_ub}
\end{proposition}
Note that the $\omega$ is convex and non-negative in $(-\infty, 1)$.

Now we are ready to provide the convergence rate of \cref{algo:rand_fw}. For the sake of clarity, we restate \cref{thm:rand_fw} with more details below. This proof is inspired from the proof techniques developed in \cite{kerdreux2018frank,zhao2023analysis}. 
\begin{theorem}
Let the initial iterate $y_0$ maps to domain of $f$, i.e.~$\cA(y_0) \in \mathrm{int}(\cK)$, and $\delta_0 = F(y_0) - \min_{y \in \cY} F(y)$,
then Randomized FW method (\cref{algo:rand_fw}) which uniformly samples $R$ elements from $\cV$ for computing the randomized LMO at each step outputs an $\varepsilon$ sub-optimal solution to the $\theta$-LHSCB problem \eqref{eq:lhscb_problem} in expectation, 
after %
\begin{align}
T_\mathrm{FW} = \Big\lceil 5.3 \frac{\widetilde{N}}{R} (\delta_0 + \theta) \max \Big\{\ln \Big(10.6 \frac{\widetilde{N}}{R} \delta_0\Big), 0 \Big\}\Big\rceil + \Big\lceil 12 \frac{\widetilde{N}}{R} \theta^2 \max \Big\{\frac1\varepsilon - \max \Big\{ \frac1{\delta_0}, 10.6 \frac{\widetilde{N}}{R} \Big\}, 0 \Big\} \Big\rceil,
\end{align}
iterations, i.e.~$\mathbb{E}[F(y_{T_\mathrm{FW}})] - \min_{y \in \cY} F(y) \leq \varepsilon$.
\label{thm:rand_fw_appendix}
\end{theorem}
\begin{proof}
First notice that since $\cA$ is a linear mapping, the update rule in \cref{algo:rand_fw} can be re-written as
\begin{equation}
v_t = \argmin_{v \in \cR_t} \langle \nabla F(y_t), v \rangle
= \argmin_{v \in \cR_t}  \langle \cA^\top (\nabla f(\cA(y_t))), v \rangle
= \argmin_{v \in \cR_t} \langle \nabla f(\cA(y_t)), \cA(v) \rangle,
\end{equation}
and
\begin{equation}
\alpha_{t} = \argmin_{\alpha \in [0, 1]} F((1-\alpha) y_t + \alpha v) 
= \argmin_{\alpha \in [0, 1]} f(\cA((1-\alpha) y_t + \alpha v)).
\end{equation}
Above also follows from the well-known affine invariance property of FW method \cite{jaggi2013revisiting}.
Next using \cref{prop:lhscb_ub} we get
\begin{align*}
F((1-\alpha) y_{t} + \alpha v_t)
&= f(\cA((1-\alpha) y_{t} + \alpha v_t)) \nonumber \\
&\leq f(\cA(y_{t})) + \langle \nabla f(\cA(y_{t})), \cA((1-\alpha) y_{t} + \alpha v_t) \rangle + \omega(\|\cA((1-\alpha) y_{t} + \alpha v_t)\|_{\nabla^2 f(\cA(y_{t}))}) \nonumber \\
&\leq F(y_{t}) + \alpha \langle \nabla f(\cA(y_{t})), \cA(v_t - y_{t}) \rangle + \omega(\alpha \|\cA(v_t - y_{t})\|_{\nabla^2 f(\cA(y_{t}))}).
\end{align*}

Let $G_t = -\langle \nabla f(\cA(y_{t})), \cA(v_t - y_{t}) \rangle$ and $D_t = \|\cA(v_t - y_{t})\|_{\nabla^2 f(\cA(y_{t}))}$. We also assume that $v_t^* \in \argmin_{v \in \cV} \langle \nabla f(\cA(y_t)), v \rangle$ is the output of the LMO if we run regular deterministic FW method (general version of \cref{algo:fw}) at iteration $t$. Note that we can also define this LMO output as $v_t^* \in \argmin_{v \in \cY} \langle \nabla f(\cA(y_t)), v \rangle$ since $\cY = \mathrm{conv}(\cV)$ and a minimizier of a linear functional on a closed convex polytope is one of its vertices. Using this definition of $v^*_t$ we also define $G_t^* = -\langle \nabla f(\cA(y_{t})), \cA(v_t^* - y_{t}) \rangle$ and $D_t^* = \|\cA(v_t^* - y_{t})\|_{\nabla^2 f(\cA(y_{t}))}$.
Then taking the minimum w.r.t.~$\alpha$ on both sides of the above inequality yields
\begin{equation}
F(y_{t+1}) = \min_{\alpha \in [0,1]}F((1-\alpha) y_{t} + \alpha v_t) \leq F(y_{t}) + \min_{\alpha \in [0,1]} [-\alpha G_t + \omega(\alpha D_t)],
\end{equation}
and by taking expectation conditioned on the randomness of $y_t$ we get
\begin{equation}
\mathbb{E}[F(y_{t+1})\,|\,y_t] \leq F(y_{t}) + \mathbb{E}[\min_{\alpha \in [0,1]} [-\alpha G_t + \omega(\alpha D_t)] \,|\,y_t].
\end{equation}
Notice that when $\alpha=0$ we obviously have that $-\alpha G_t + \omega(\alpha D_t) = 0$ since $\omega(0) = 0$. Therefore, $\min_{\alpha \in [0,1]} [-\alpha G_t + \omega(\alpha D_t)] \leq 0$. Note that, in contrast to regular FW method (\cref{algo:fw}) where $G_t$ is always negative, for the worst case sampling of $\cR_t$, the quantity $G_t$ can be positive. Therefore, choosing the $\alpha$ dynamically is necessary to ensure non-increasing objective values in randomized FW \cite{kerdreux2018frank}.
Next, we upper bound the above inequality as
\begin{align}
\mathbb{E}[F(y_{t+1})\,|\,y_t] &\leq F(y_{t}) + (1-\mathrm{Prob}(v_t^* \in \cV_t)) \cdot 0 + \mathrm{Prob}(v_t^* \in \cV_t) \min_{\alpha \in [0,1]} [-\alpha G_t^* + \omega(\alpha D_t^*)] \nonumber \\
&= F(y_{t}) + (R/\widetilde{N}) \min_{\alpha \in [0,1]} [-\alpha G_t^* + \omega(\alpha D_t^*)].
\end{align}

Now notice that the above ``descent lemma'' inequality is similar to \cite[Inequality 2.7]{zhao2023analysis}, except for the conditional expectation on the LHS and the discount factor $(R/\widetilde{N}) \leq 1$ on second term which corresponds to the minimum decrease in objective value. Therefore, rest of the proof follow similar arguments as \cite[Theorem 1]{zhao2023analysis}.
\end{proof}
This proves that randomized FW method convergence to the the minimizer of the LHSCB problem in $((\widetilde{N}/R)\varepsilon^{-1})$ steps. Compared to regular FW method, randomized version increases the iteration complexity by a factor of $\widetilde{N}/R$. We see similar increase in iteration complexity even for convex problems with bounded Lipschitz smoothness \cite{kerdreux2018frank}.
In the next section we specialize this result to provide a convergence rate guarantee for \ouralgo (\cref{algo:ouralgo}).

\subsection{Theoretical analysis of \ouralgo}
Again for the sake of clarity we first restate \cref{cor:ouralgo} providing the iteration complexirt of \ouralgo (\cref{algo:ouralgo_subroutines}) with more details and then provide its proof.
\begin{corollary}[of \cref{thm:rand_fw_appendix}]
Let $V_{\pi_0}$ be full-rank, $\gamma=0$, $R \leq {N \choose K}$, $\alpha_\mathrm{tol}$ be small enough and $\delta_0 = \max_{\pi \in \Delta^{\mathcal{S}}} g(\pi) - g(\pi_0)$,
then \ouralgo (\cref{algo:ouralgo}) which uniformly samples $R$ subsets from $\mathcal{S}$ for computing the randomized LMO at each step, outputs an $\varepsilon$ sub-optimal solution to the D-optimal design problem \eqref{eq:optimal design} in expectation, 
after %
\begin{align}
T_\mathrm{od} = \Big\lceil 5.3 \frac{{N \choose K}}{R} (\delta_0 + d) \max \Big\{\ln \Big(10.6 \frac{{N \choose K}}{R} \delta_0\Big), 0 \Big\}\Big\rceil + \Big\lceil 12 \frac{{N \choose K}}{R} d^2 \max \Big\{\frac1\varepsilon - \max \Big\{ \frac1{\delta_0}, 10.6 \frac{{N \choose K}}{R} \Big\}, 0 \Big\} \Big\rceil
\end{align}
iterations, i.e.~$\max_{\pi \in \Delta^{\mathcal{S}}} g(\pi) - \mathbb{E}[g(\pi_{T_\mathrm{od}})] \leq \varepsilon$.
\label{cor:ouralgo_appendix}
\end{corollary}
\begin{proof}
Let $\mathbb{S}_{+}^d$ be the regular cone of all PSD matrices of dimension $d \times d$. Then we notice that the maximization problem in the D-optimal design \eqref{eq:optimal design} can be written as a special case of the minimization problem \eqref{eq:lhscb_problem_appendix} by setting $F = -g$, $y = \pi$, $\cK = \mathbb{S}_{+}^d$, $f = -\log\det: \mathbb{S}_{+}^d \to \mathbb{R}$, $\cA=\sum_{S \in \cS} \pi(S) A_S A_S^\top: \pi \to \mathbb{S}_{+}^d$, and $\cV = \cS$ ($\widetilde{N} = {N \choose K}$). It is known that $f = \log\det$ is an LHSCB function (\cref{assume:lhscb}) with $\theta=d$ \cite{zhao2023analysis}. Moreover, $\cA$ maps to $\mathbb{S}_{+}^d$ because $\pi(S) A_S A_S^\top$ is PSD for all $S \in \cS$.

This equivalence implies that when $\alpha_{\mathrm{tol}} \to 0$ and $\gamma=0$, running \ouralgo (\cref{algo:ouralgo}) is equivalent to running randomized FW method (\cref{algo:rand_fw}) for the above specialization. Please see \cref{sec:scalable algorithm} for description of how various sub-routines of \cref{algo:ouralgo} implements varioius steps of \cref{algo:rand_fw}.
Finally, we can satisfy the condition $y_0 = \mathrm{int}(\cK)$ if $V_{\pi_0}$ is full rank since $V_{\pi_0} \in \mathrm{int}(\mathbb{S}_{+}^d)$. Then the the iteration complexity of \ouralgo directly follows from \cref{thm:rand_fw_appendix}.
\end{proof}

Note that above theorem requires that $\alpha_{\mathrm{tol}}$ is small enough, but since \goldensearch is an exponentially fast algorithm, $\alpha_{\mathrm{tol}}$ can reach machine precision in a very few steps. For example in our implementation we set $\alpha_{\mathrm{tol}} = 10^{-16}$ and this is achieved in $76$ steps of \goldensearch.

\section{Additional algorithmic details and sub-routines for \ouralgo}
\label{app:algo_additional_details}

In this section, we provide additional details for \ouralgo (\cref{algo:ouralgo}). We begin by providing the complete pseudocode for the three undefined sub-routines of \ouralgo (\updateinverse, \goldensearch, \updatelogdet) in \cref{algo:ouralgo_subroutines}. Next we provide explanation for various algorithmic choices made in \ouralgo.

\begin{algorithm}[t]
\caption{Sub-routines required for \cref{algo:ouralgo}}
\label{algo:ouralgo_subroutines}

\SetKwFunction{FMain}{\updateinversenox}
\SetKwProg{Fn}{Sub-routine}{:}{}
\Fn{\FMain{$V^{\mathrm{inv}}, A, \alpha$}\label{algo-step:update_inverse_subroutine}}{
    Set $r={K \choose 2}$, $\widetilde{V}^{\mathrm{inv}} = {V}^{\mathrm{inv}}/(1-\alpha)$, and $\widetilde{A} = \sqrt{\alpha} A$ 
    
    $V_+^{\mathrm{inv}} = \widetilde{V}^{\mathrm{inv}} - V^{\mathrm{inv}} \widetilde{A} \bigg(\mathbf{I}_{r \times r} + \widetilde{A}^T  V^{\mathrm{inv}} \widetilde{A} \bigg)^{-1} \widetilde{A}^T V^{\mathrm{inv}}$ 

    \Return $V_+^{\mathrm{inv}}$
}

\SetKwFunction{FMain}{\goldensearchnox}
\SetKwProg{Fn}{Sub-routine}{:}{}
\Fn{\FMain{$V^{\mathrm{inv}}, A, \alpha_\mathrm{tol}$}\label{algo-step:golden_search_subroutine}}{
    Set $\varphi = (\sqrt{5}+1)/2, \alpha_a = 0, \alpha_h = \alpha_b = 1$ 

    Initialize: \\
    $\alpha_c = \alpha_a + \alpha_h \varphi^{-2}, \alpha_d = \alpha_a + \alpha_h\varphi^{-1}$  \\
    $V_c^{\mathrm{ld}} = \text{\updatelogdet}(V^{\mathrm{inv}}, A, \alpha_c)$  \\
    $V_d^{\mathrm{ld}} = \text{\updatelogdet}(V^{\mathrm{inv}}, A, \alpha_d)$
    
    \While{$|\alpha_a - \alpha_b| \geq \alpha_\mathrm{tol}$}{
        $\alpha_h = \alpha_h \varphi^{-1}$ 

        \If{$V_c^{\mathrm{ld}} > V_d^{\mathrm{ld}}$}{
            $\alpha_b, \alpha_d, V_d^{\mathrm{ld}}, \alpha_c = \alpha_d, \alpha_c, V_c^{\mathrm{ld}}, \alpha_a + \alpha_h \varphi^{-2}$

            $V_c^{\mathrm{ld}} = \text{\updatelogdet}(V^{\mathrm{inv}}, A, \alpha_c)$ 
        }
        \Else{
            $\alpha_a, \alpha_c, V_c^{\mathrm{ld}}, \alpha_d = \alpha_c, \alpha_d, V_d^{\mathrm{ld}}, \alpha_a + \alpha_h \varphi^{-1}$

            $V_d^{\mathrm{ld}} = \text{\updatelogdet}(V^{\mathrm{inv}}, A, \alpha_d)$ 
        }
    }

    \Return $(\alpha_a+\alpha_b)/2$
}

\SetKwFunction{FMain}{\updatelogdetnox}
\SetKwProg{Fn}{Sub-routine}{:}{}
\Fn{\FMain{$V^{\mathrm{inv}}, A, \alpha$}
\label{algo-step:inverse_update_subroutine}}{
    Set $r={K \choose 2}$, $\widetilde{V}^{\mathrm{inv}} = {V}^{\mathrm{inv}}/(1-\alpha)$, and $\widetilde{A} = \sqrt{\alpha} A$
    
    $V_+^{\mathrm{ld}} = d\log(1-\alpha) + \log\det\Big(\mathbf{I}_{r \times r} + \widetilde{A}^\top \widetilde{V}^{\mathrm{inv}} \widetilde{A} \Big)$ 

    \Return $V_+^{\mathrm{ld}}$
}

\end{algorithm}

\subsection{Low-rank update for the $\log\det$ objective}
\label{app:logdet_update}

As mentioned in \cref{sec:line_search} computing objective function values involves computing the $\log\det$ of a $d \times d$ matrix which has a worst case complexity of $O(d^3)$. However, we simplify this computation by noting that $V_{\pi}$ is only updated with the $r={K \choose 2}$ rank matrix $A_{S} A_{S}^\top$ (for some $S$). Since $A_{S} \in \mathbb{R}^{r \times d}$ this is a low-rank matrix when $r \ll d$, which allows us to simplify the relevant change in objective function value after the convex combination with a step size $\alpha$ as follows. For the sake of simplicity, we denote $\widetilde{V}^{\mathrm{inv}} = (V_{\pi})^{-1}/(1-\alpha)$ and $\widetilde{A} = \sqrt{\alpha} A_{S}$. Then,
\begin{align}
&g((1-\alpha) \cdot \pi + \alpha \cdot e_{S}) - g(\pi) \nonumber \\
&= \log\det\big((1-\alpha) V_{\pi} + \alpha A_{S} A_{S}^\top \big) - \log\det(V_{\pi}) \nonumber \\
&= \log\det\big((1-\alpha)(\mathbf{I}_{d \times d} + \widetilde{A} \widetilde{A}^\top \widetilde{V}^{\mathrm{inv}})V_{\pi}\big) - \log\det(V_{\pi}) \nonumber \\
&= d \log(1-\alpha) + \log\det(\mathbf{I}_{d \times d} + \widetilde{A} \widetilde{A}^\top \widetilde{V}^{\mathrm{inv}} ) \nonumber \\
&= d \log(1-\alpha) + \log\det(\mathbf{I}_{r \times r} +\widetilde{A}^\top \widetilde{V}^{\mathrm{inv}} \widetilde{A} ), \label{eq:delta_logdet}
\end{align}
where the third equality follows from the facts that $\log\det(BC) = \log\det(B) + \log\det(C)$ and $\log\det(cB_{d \times d}) = d\log(c) + \log\det(B)$, and the third equality follows from the Weinstein–Aronszajn identity $\log\det(\mathbf{I} + BC) = \log\det(\mathbf{I} + CB)$ \cite{pozrikidis2014introduction}. \goldensearch sub-routine (\cref{algo-step:golden_search_subroutine} of \cref{algo:ouralgo_subroutines}) in turn uses the \updatelogdet sub-routine (\cref{algo-step:inverse_update_subroutine} of \cref{algo:ouralgo_subroutines}) which implements \eqref{eq:delta_logdet} assuming the access to $V_{\pi}^{-1}$. Since \eqref{eq:delta_logdet} computes $\log\det$ of an $r \times r$ matrix we reduce complexity of computing change in objective function values from $O(d^3)$ to $O(r^3 + rd^2)$. Note that one can further reduce the total complexity of \goldensearch from $O(\log_{\varphi}(\alpha_{\mathrm{tol}}^{-1}) (r^3 + rd^2))$ to $O(\log_{\varphi}(\alpha_{\mathrm{tol}}^{-1}) + r^3 + rd^2)$ by computing eigenvalues $\{\lambda_i\}$ of $A_{S_t}^\top V_{\pi_t}^{-1}A_{S_t}$ once and then computing the change in objective function values as $d \log(1-\alpha) + \sum_i^r \log(1+\alpha \lambda_i)$, however this approach may be numerically less stable.

\subsection{Low-rank inverse update}
\label{sec:inverse_update}

Here we expand on how \ouralgo computes the inverse of the ${d \times d}$ matrix $V_{\pi_t}$ used in \partialgrad (\cref{algo-step:partial_grad_subroutine} of \cref{algo:ouralgo}) and \updatelogdet (\cref{algo-step:inverse_update_subroutine} of \cref{algo:ouralgo_subroutines}) sub-routines. Naively implementing it at every iteration incurs a cost of $O(d^3)$ every iteration.
Instead, \updateinverse sub-routine (\cref{algo-step:update_inverse_subroutine} of \cref{algo:ouralgo_subroutines}) is used in \ouralgo to iteratively update the matrix $V^{\mathrm{inv}}$ storing the inverse as follows
\begin{align}
V^{\mathrm{inv}}_+ &= ((1-\alpha) V_{\pi} + \alpha A_{S} A_{S}^\top)^{-1} \nonumber \\
&= ((\widetilde{V}^{\mathrm{inv}})^{-1} + \widetilde{A} \widetilde{A}^\top)^{-1} \\
&= \widetilde{V}^{\mathrm{inv}} - \widetilde{V}^{\mathrm{inv}} \widetilde{A} (\mathbf{I}_{r \times r} + \widetilde{A}^\top \widetilde{V}^{\mathrm{inv}} \widetilde{A})^{-1} \widetilde{A}^\top \widetilde{V}^{\mathrm{inv}}, \nonumber
\end{align}
where we used the notations from \cref{sec:line_search} and the second equality used the Woodburry matrix inversion identity \cite{guttman1946enlargement,woodbury1950inverting} with $r={K \choose 2}$. When $r \ll d$, this update rule improves the complexity of finding the inverse to $O(r^3 + rd^2)$. Note that one of the inputs to \ouralgo (\cref{algo:ouralgo}) is $ V^{\mathrm{inv}}_0$ which has an amortized cost of $O(d^3/T_{\mathrm{od}})$.

\subsection{Sparse iterate update}
\label{sec:iterate_update}

Finally we discuss how we store and update the large ${N \choose K}$ dimensional iterate $\pi_t$ (\cref{algo-step:update-iterate} of \cref{algo:ouralgo}). If we initialize \ouralgo with a sparse $\pi_0$ and if $T_{\mathrm{od}} \ll {N \choose K}$, then its iterates are sparse since the update rule $\pi_{t+1} = (1-\alpha_t) \cdot \pi_t + \alpha_t \cdot e_{S_t}$ consists of the one-hot probability vector $e_{S_t}$. This implies that maintaining $\pi_{t}$ as a ${N \choose K}$ dimensional dense vector is unnecessary and expensive. Therefore, \ouralgo maintains it as sparse vector $\pi_t^{(sp)}$ (\texttt{scipy.sparse.csc\_array}\footnote{\url{https://scipy.org/}}) of size ${N \choose K}$. This begs the question how \ouralgo maps a subset $S \in \mathcal{S}$ to an index of $\pi_t^{(sp)}$. \ouralgo answers this by using the combinadics number system of order $K$\footnote{\url{https://en.wikipedia.org/wiki/Combinatorial_number_system}}  which defines a bijective mapping, $\mathcal{C}_K: \mathcal{S} \to [{N \choose K}]$, from the collection of subsets of size $K$ to integers. Therefore our update rule in terms of $\pi_t^{(sp)}$ translates to $\pi_{t+1}^{(sp)} = (1-\alpha_t) \cdot \pi_t^{(sp)} + \alpha_t \cdot e_{\mathcal{C}_K(S_t)}^{(sp)}$, where $e_{\mathcal{C}_K(S_t)}^{(sp)}$ is the basis vector for the $\mathcal{C}_K(S_t)$-th dimension. These modifications improves the worst case complexity of updating and maintaining $\pi_t$ after $t$ iterations from $O({N \choose K})$ to only $O(t)$.

\section{Extended Experiments}
\label{app:experiments}

In this section we provide additional experimental details and results.

\subsection{Synthetic Feedback Setup}
\label{app:synthexperiments}

In Figure~\ref{fig:synthndcg}, we show the mean NDCG metric (higher is better) for the experiment in Section~\ref{ssec:synthetic} for $K=2,3$ on the three datasets. The observations are consistent with Section~\ref{ssec:synthetic}, where \ouralgo achieves better ranking performance than baselines.

\begin{figure}[h]
    \centering
    \subfigure[BEIR-COVID, $K=2$]{
        \includegraphics[width=0.31\textwidth]{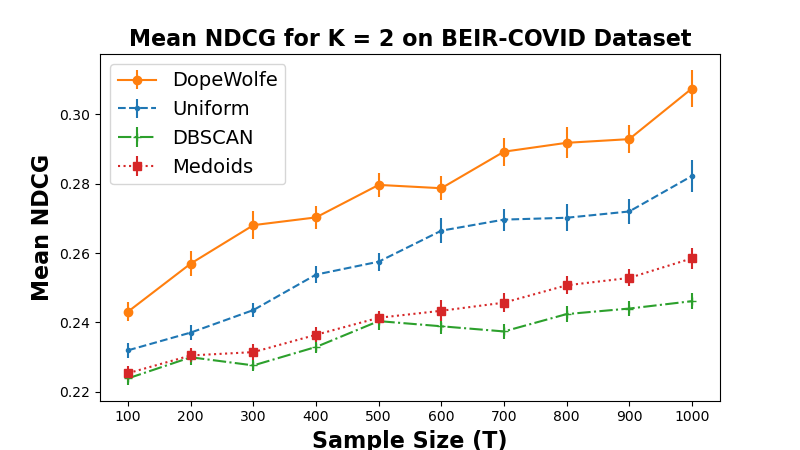}
        \label{fig:subfig2_ndcg}
    }
    \subfigure[TREC-DL, $K=2$]{
        \includegraphics[width=0.31\textwidth]{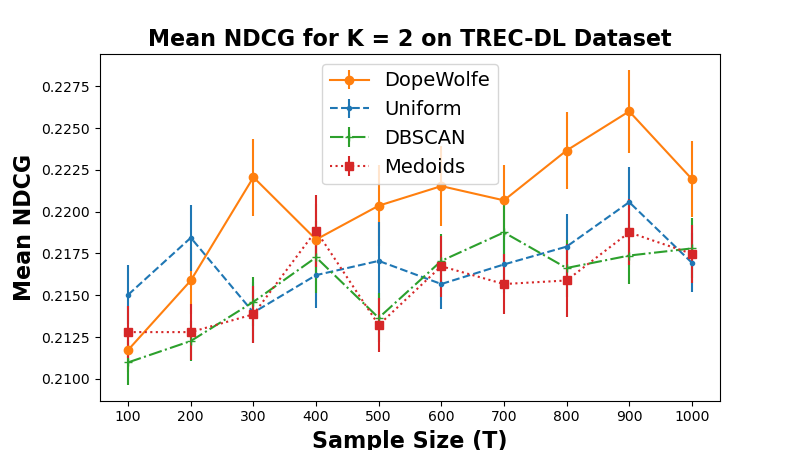}
        \label{fig:subfig1_ndcg}
    }
    \subfigure[NECTAR, $K=2$]{
        \includegraphics[width=0.31\textwidth]{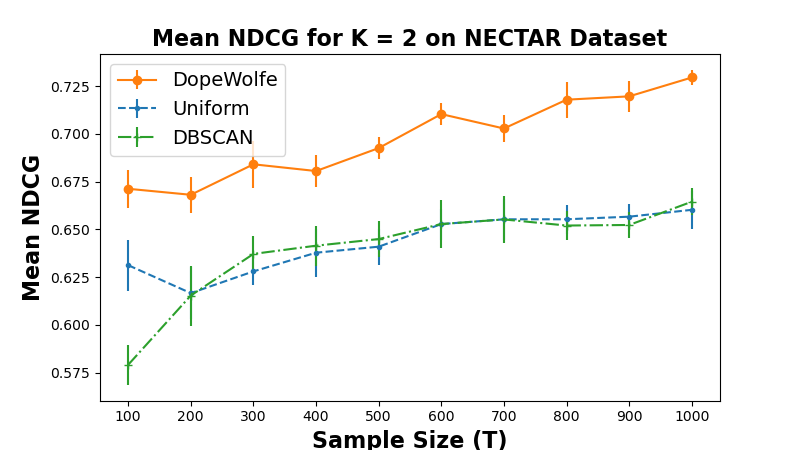}
        \label{fig:subfig3_ndcg}
    }
    \subfigure[BEIR-COVID, $K=3$]{
        \includegraphics[width=0.31\textwidth]{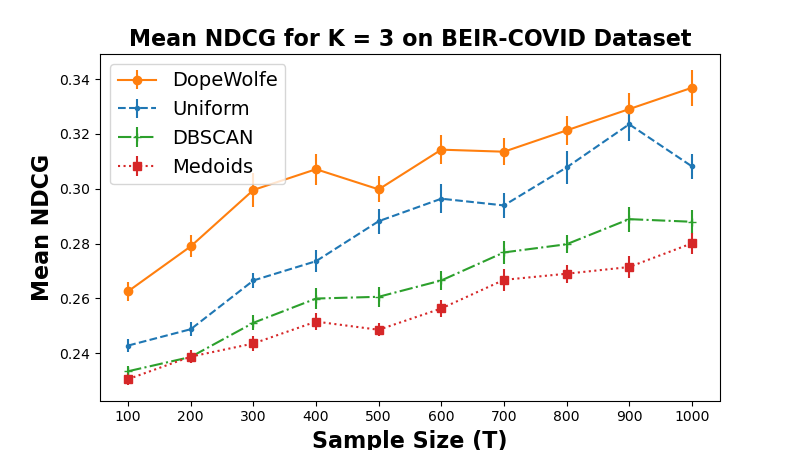}
        \label{fig:subfig4_ndcg}
    }
    \subfigure[TREC-DL, $K=3$]{
        \includegraphics[width=0.31\textwidth]{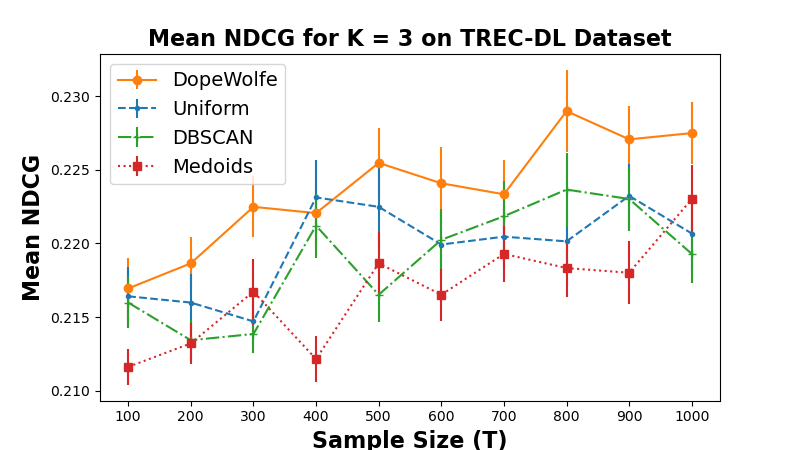}
        \label{fig:subfig5_ndcg}
    }
    \subfigure[NECTAR, $K=3$]{
        \includegraphics[width=0.31\textwidth]{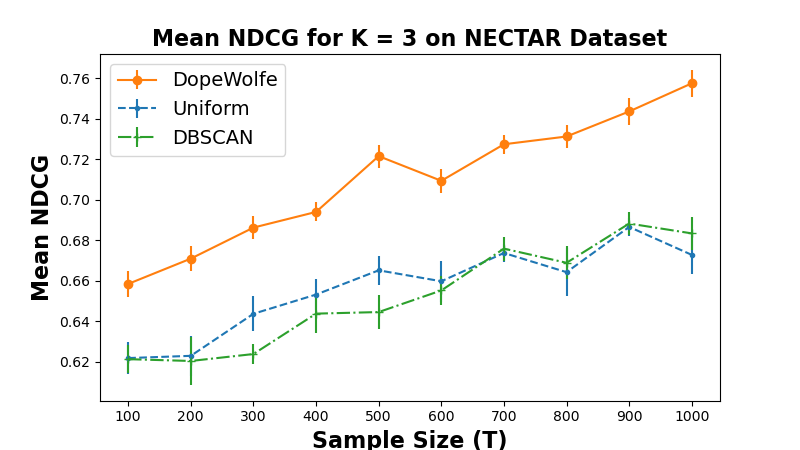}
        \label{fig:subfig6_ndcg}
    }
    \caption{Mean NDCG metric (higher is better) on BEIR-COVID, TREC-DL and NECTAR datasets with synthetic feedback.}
    \label{fig:synthndcg}
\end{figure}

We next show the ranking loss~\eqref{eq:ranking loss} (lower is better) and NDCG metric (higher is better) for $K=4$ for BEIR-COVID and TREC-DL datasets in Figure~\ref{fig:synthk4} for the synthetic setup (Section~\ref{ssec:synthetic}. Note that, for NECTAR dataset, since the number of answers for each question are only seven, the number of possible subsets are same for $K=3$ and $K=4$; thus, we ignore NECTAR dataset for $K=4$. We again observe that \ouralgo achieves better performance than baselines on the ranking task.

\begin{figure}[t]
    \centering
    \subfigure[BEIR-COVID, $K=4$]{
        \includegraphics[width=0.31\textwidth]{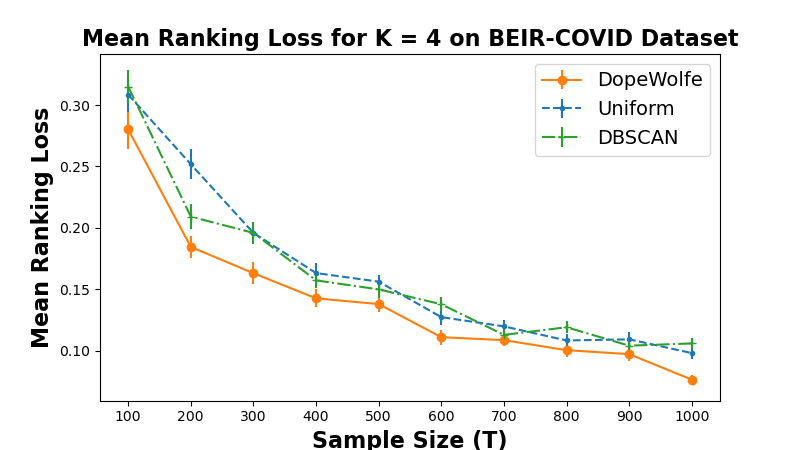}
        \label{fig:subfig2_k4}
    }
    \subfigure[BEIR-COVID, $K=4$]{
        \includegraphics[width=0.31\textwidth]{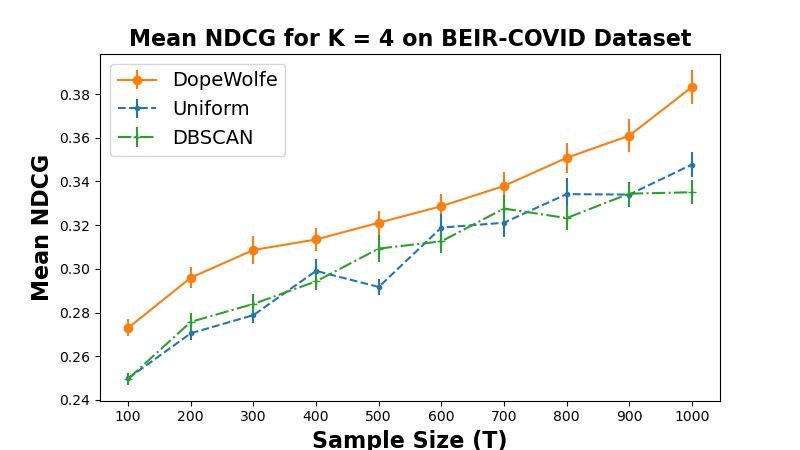}
        \label{fig:subfig1_k4}
    } \\
    \subfigure[TREC-DL, $K=4$]{
        \includegraphics[width=0.31\textwidth]{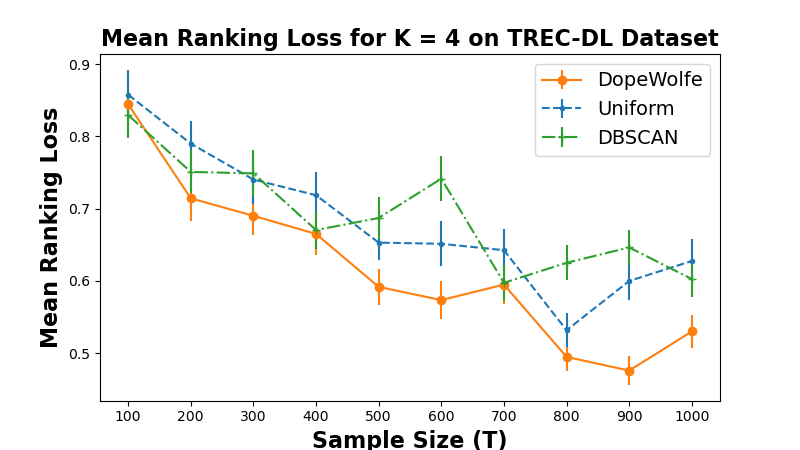}
        \label{fig:subfig3_k4}
    }
    \subfigure[TREC-DL, $K=4$]{
        \includegraphics[width=0.31\textwidth]{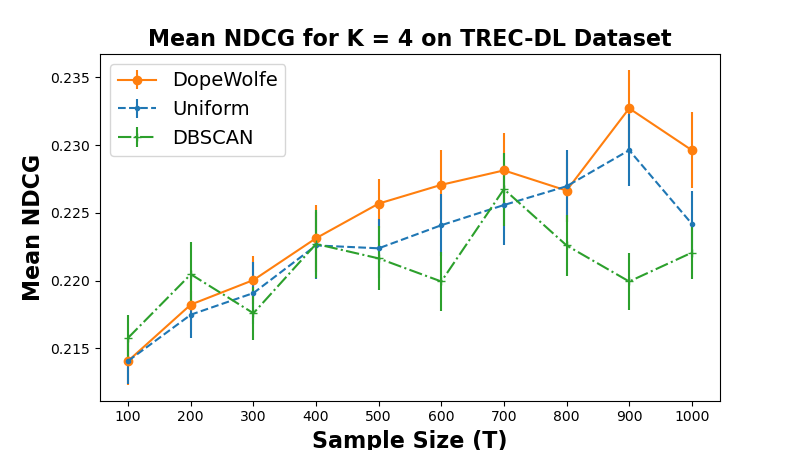}
        \label{fig:subfig4_k4}
    }
    \caption{BEIR-COVID and TREC-DL datasets with synthetic feedback: Mean KDtau and NDCG metric on Datasets for $K=4$. On average, \ouralgo achieves better performance than Uniform sampling and DBSCAN.}
    \label{fig:synthk4}
\end{figure}

\subsection{Real Feedback Setup}
\label{app:realexperiments}

In this section, we provide more details on the experimental choices for the experiment with real feedback (Section~\ref{ssec:real}). For both BEIR-COVID and TREC-DL datasets, since the feedback is real, we use more representative 1024 dimensional BGE-M3 embedding~\cite{bge-m3} to ensure that the features are rich and reasonable enough to capture the true ranking (realizable ranking problem). We then apply Principal Component Analysis (PCA)~\cite{abdi2010principal} so that the feature matrix is full rank. We end up with 98 and 78 features for BEIR-COVID and TREC-DL datasets, respectively. We then run \ouralgo on the datasets by setting $\gamma=10^{-6}$, $\alpha_{\mathrm{tol}}=10^{-16}$, and $R = \min({100 \choose K}, 10^5)$ and initializing it with a uniformly chosen single subset, i.e.~$\texttt{nnz}(\pi_0) = 1$. We find that the \ouralgo converges in 200 iterations for BEIR-COVID and in 1500 iterations for TREC-DL. Since the feedback is real and does not change whenever we sample a new $K$-way subset for feedback, we pick the top $K$-way subsets that have the highest probability mass in the sampling distribution obtained through \ouralgo. Note that due to significant amount of ranking ties in the BEIR-COVID dataset, we use the MLE for the pairwise rank breaking of the observations \cite[Equation 39, Page 25]{negahban2018learning} instead of the $K$-wise MLE \eqref{eq:ranking loglik}. We found that choosing  only top-25 $K$-way subsets for BEIR-COVID  works well, possibly due to the simpler nature of the problem as a consequence of the ties. For TREC-DL dataset, we choose top-500 $K$-way subsets. Lastly, we run the Plackett-Luce model's maximum likelihood estimation process using gradient descent with Barzilai-Borwein stepsize rule \cite{barzilai1988two} with min and max stepsize clip values of $10^{-8}$ and $5\times10^{4}$/$10^{3}$ for BEIR-COVID/TREC-DL, which converges in 200/1000 iterations for BEIR-COVID/TREC-DL. All the reported metrics provide mean and standard error over over 10 trials.

The ranking loss~\eqref{eq:ranking loss} results for this experiment is shown in Figure~\ref{fig:real_feedback_exp}. We also show NDCG@10 for this experiment in Figure~\ref{fig:real_feedback_ndcg}. We choose NDCG@10 to show the accuracy at the top of the ranking. In the NDCG computation, we use linear gain function for BEIR-COVID, and an exponential gain with temperature 0.1 for TREC-DL with the aim of differentiating its very close score values. We see that \ouralgo performs better than the baselines for the ranking tasks in both datasets. 

\begin{figure}[t]
    \centering
    \subfigure[BEIR-COVID]{
        \includegraphics[width=0.31\textwidth]{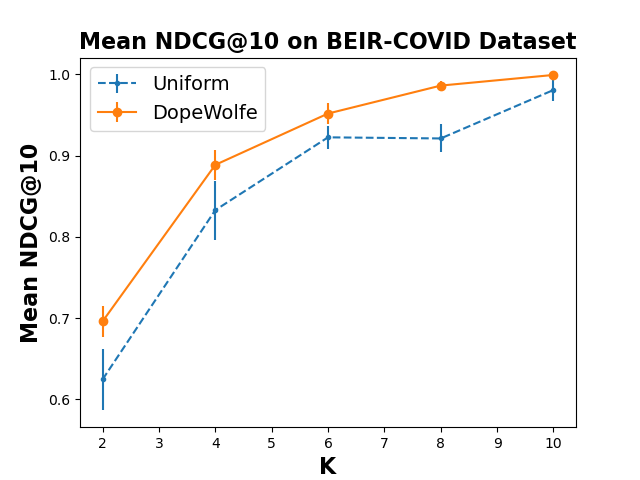}
        \label{fig:real_dl_subfig1_ndcg}
    }
    \subfigure[TREC-DL]{
        \includegraphics[width=0.31\textwidth]{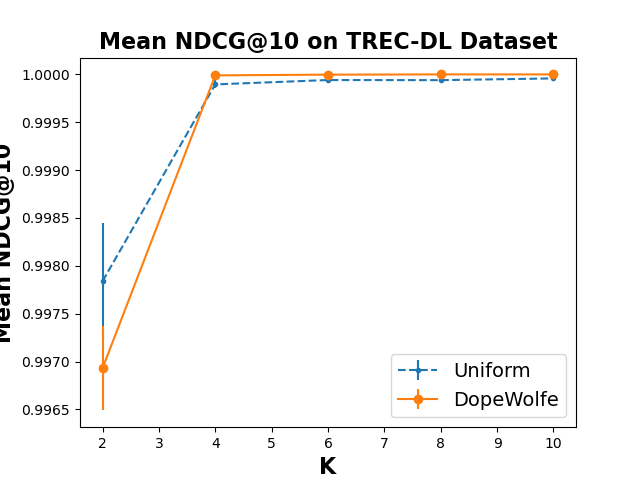}
        \label{fig:real_dl_subfig2_ndcg}
    }
    \vspace{1cm}
    \caption{BEIR-COVID and TREC-DL datasets with real feedback: Models learned through ranking feedback collected on \ouralgo samples achieve higher NDCG@10 than the ones learned on Uniformly selected samples.}
    \label{fig:real_feedback_ndcg}
\end{figure}

\end{document}